\documentclass{article} 
\usepackage{iclr2026_conference,times}

\usepackage{amsmath,amsfonts,bm}

\def\eqref#1{equation~\ref{#1}}

\def\1{\bm{1}}

\DeclareMathAlphabet{\mathsfit}{\encodingdefault}{\sfdefault}{m}{sl}
\SetMathAlphabet{\mathsfit}{bold}{\encodingdefault}{\sfdefault}{bx}{n}

\usepackage{hyperref}
\usepackage{url}

\usepackage{array}   
\usepackage{booktabs} 
\usepackage{graphicx}
\usepackage{algorithm}
\usepackage{algorithmic}
\usepackage{amsthm,amsmath,amssymb}
\usepackage{mathrsfs}
\usepackage{multirow}   
\usepackage[table]{xcolor} 
\usepackage{colortbl}
\usepackage{rotating} 
\usepackage{float}

\usepackage{wrapfig}
\usepackage{subfig}
\usepackage{bbding}
\usepackage{caption}
\usepackage{titletoc}

\newtheorem{theorem}{Theorem}
\newtheorem{corollary}{Corollary}

\newtheorem{remark}{Remark}[section]
\newtheorem{definition}{Definition}[section]

\newenvironment{myquotation}{\setlength{\leftmargini}{5em}\quotation}{\endquotation}

\makeatletter
\def\thanks#1{\protected@xdef\@thanks{\@thanks
        \protect\footnotetext{#1}}}
\makeatother

\title{PACT: Preserving Anchored Cores in Task-vectors for Model Merging}

\author{
Ningyuan Shi$^{1,\star}$, \quad Zhipeng Zhou$^{2,\star}$, \quad Hao Wang$^{3}$, \quad Chunyan Miao$^{2}$, \quad Peilin Zhao$^{1,\dagger}$\\
$^{1}$Shanghai Jiao Tong University\thanks{$\star$ Equal contribution. $\dagger$ Corresponding author.} \quad 
$^{2}$Nanyang Technological University \\  
$^{3}$The Hong Kong University of Science and Technology (Guangzhou) \\  
{\ttfamily\small shiningyuanAccount@sjtu.edu.cn, zzpustcml@gmail.com} \\
{\ttfamily\small haowang@hkust-gz.edu.cn, ascymiao@ntu.edu.sg, peilinzhao@sjtu.edu.cn}
}

\definecolor{lightgreen}{RGB}{229,251,229}

\usepackage{acro}
\DeclareAcronym{Ours}{
  short = \texttt{PACT},
  long  = \textbf{P}reserve \textbf{A}nchored \textbf{C}ores}
  
    \usepackage{xspace}
\newcommand{\ours}[0]{\texttt{PACT}\xspace}

\iclrfinalcopy 
\begin{document}

\maketitle

\begin{abstract}
Model merging has emerged as a training-free alternative to multi-task learning, aiming to combine multiple task-specific fine-tuned models into a single multi-task model. Most existing model merging approaches follow the \emph{Task Arithmetic} paradigm, which decomposes fine-tuned weights into pre-trained parameters and task vectors, and performs merging exclusively in the task-vector space. The effectiveness of this paradigm implicitly relies on the assumption that task-specific knowledge is encoded solely within task vectors. In this paper, we argue that this assumption generally does not hold due to the intrinsic task preferences of pre-trained models. Specifically, we identify \textbf{Load-Bearing Wall (LBW) dimensions}, namely some task-critical knowledge that remains embedded in the pre-trained weights rather than being fully transferred into task vectors. We characterize LBW dimensions from both scalar-weight and subspace perspectives, thereby covering the major paradigms of existing model merging methods. Our analysis reveals that, by ignoring LBW dimensions, task-vector-based approaches fail to fully resolve task conflicts and may inadvertently damage task-specific knowledge encoded in the pre-trained model, leading to degradation. To address this issue, we propose \ac{Ours}, which preserves the anchored task-specific cores (i.e., LBW components) within task vectors by aligning their orthogonal complements with the subspace of the pre-trained weights. These aligned subspace components are then removed from the task vectors before applying existing model merging algorithms. Furthermore, we develop an efficient variant based on randomized SVD to improve scalability. Since \ac{Ours} approaches model merging from a complementary perspective, it can be seamlessly integrated with existing methods. Extensive experiments across multiple benchmarks demonstrate that \ac{Ours} consistently enhances mainstream model merging approaches and establishes new state-of-the-art performance. Code is available at \url{https://github.com/dajvmasonjason-glitch/pact}.
\end{abstract}

\section{Introduction}
Pre-trained models (PTMs) have become the foundation of modern machine learning systems, providing strong general-purpose representations across a wide range of domains and tasks~\citep{carion2020end,radford2021learning,caron2021emerging}. To achieve high performance on downstream applications, PTMs are typically specialized through task-specific fine-tuning~\citep{wortsman2022robust,ilharco2022patching}. However, extending such specialization to the multi-task setting often requires complex optimization procedures, substantial computational resources, and access to large amounts of task-specific training data~\citep{wei2024task}. Model merging has recently emerged as an attractive training-free alternative, aiming to combine multiple task-specific fine-tuned models into a single multi-task model without additional training~\citep{li2023deep,yang2026model}.

Existing model merging methods can be broadly categorized into scalar weight-based and subspace-based approaches~\citep{yang2026model,ruan2025task}. Scalar weight-based methods~\citep{ilharco2022editing,yadav2023ties,nguyen2025regmean++} treat individual parameter coordinates as the basic unit of task knowledge and perform coordinate-wise operations on model weights or task vectors. In contrast, subspace-based approaches~\citep{stoica2025model,gargiulo2025task,marczak2025no} represent task knowledge as a collection of important directions in parameter space and merge models by identifying, preserving, or recombining task-relevant subspaces. Despite their methodological differences, both paradigms operate primarily in the space of \emph{task vectors}, namely the parameter updates induced by fine-tuning.

The widespread adoption of task vectors is largely motivated by the insights of Task Arithmetic~\citep{ilharco2022editing}, which showed that task vectors corresponding to different tasks are often approximately orthogonal. This observation suggests that task-specific knowledge is largely encoded in task vectors and can therefore be manipulated independently. Consequently, most existing model merging methods implicitly assume that task vectors provide a sufficiently complete representation of task-specific knowledge, while the pre-trained model serves merely as a shared initialization point.
In this paper, we challenge this assumption by asking a simple yet fundamental question:

\begin{myquotation}
\vskip -0.05in
\emph{Is task-specific and -critical knowledge fully captured by task vectors?}
\vskip -0.05in
\end{myquotation}

Our answer is negative. Through a series of controlled experiments spanning both scalar weight-based and subspace-based merging paradigms, we identify a previously overlooked phenomenon: a portion of task-specific and -critical knowledge remains anchored in the pre-trained model itself. We refer to such dimensions as \textbf{Load-Bearing Wall (LBW)} dimensions. Although these dimensions often exhibit only small changes during fine-tuning and therefore contribute little to the task vector, they play a disproportionately important role in maintaining task performance. As a result, task vectors alone provide an incomplete representation of task-specific knowledge.
\begin{wrapfigure}[15]{r}{0.68\textwidth}
   \centering
   \includegraphics[width=\linewidth]{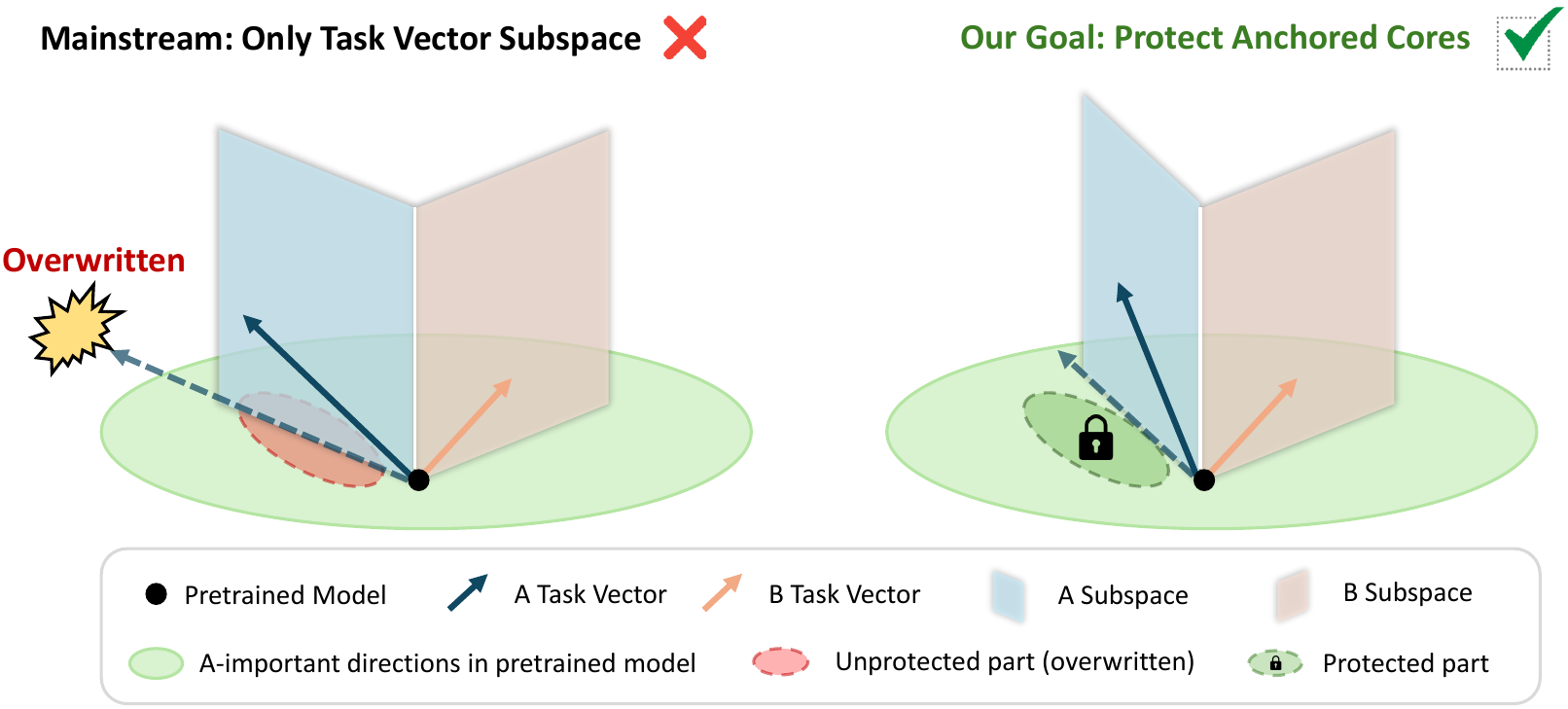} 
   \caption{Illustration of LBW dimensions in pre-trained model and its impact on model merging.} 
   \label{fig:intro} 
\end{wrapfigure}

This observation has important implications for model merging. Since existing approaches operate exclusively on task vectors, they are unable to explicitly protect LBW dimensions during merging. Consequently, updates introduced by other task vectors may overwrite or distort the pre-trained structures on which a task critically depends, leading to performance degradation after merging. Through destructive perturbation experiments, we verify both the existence of LBW components and their vulnerability during model aggregation, as illustrated in Figure~\ref{fig:intro}.

Building upon these findings, we propose \ours, a data-free framework for protecting task-specific LBW knowledge during model merging. The key idea is to identify and preserve the task-anchored components associated with LBW dimensions while filtering out task-vector components that interfere with the corresponding pre-trained subspace. Since \ours operates directly on task vectors, it can be seamlessly integrated with existing scalar weight-based and subspace-based merging algorithms. To further improve scalability, we develop an efficient variant based on randomized SVD with substantially reduced computational complexity.
Our contributions are summarized as follows:

\begin{itemize}
\item We revisit a fundamental assumption underlying Task Arithmetic and investigate whether task vectors fully capture task-specific knowledge. We empirically identify the existence of LBW dimensions and demonstrate their critical role in model merging.

\item We propose \ours, a data-free framework that protects task-specific LBW knowledge from interference during model merging. We further develop an efficient variant with improved scalability through randomized SVD.

\item Extensive experiments on vision and language benchmarks, covering both full fine-tuning and LoRA fine-tuning settings, demonstrate that \ours consistently improves mainstream model merging methods and establishes new state-of-the-art (SOTA) performance.
\end{itemize}

\section{Related Work}
Model merging aims to integrate multiple task-specific fine-tuned expert models into a single multi-task model without accessing their original training data, thereby avoiding the computational overhead and data privacy concerns of traditional multi-task learning. The task vector lies at the core of most static merging methods. Task Arithmetic~\citep{ilharco2022editing} performs multi-task merging by scaling and summing task vectors, establishing the plug-and-play merging paradigm. TIES-Merging~\citep{yadav2023ties} sequentially applies magnitude pruning, sign disambiguation, and disjoint aggregation to alleviate parameter conflicts. Consensus Merging~\citep{wang2024localizing} eliminates ``selfish'' and ``catastrophic'' weights that benefit only a single task at the expense of others. Concrete Merging~\citep{tang2023concrete} uses meta-learning to generate binary masks that suppress conflicting parameters. AWD~\citep{xiong2024multi} reduces interference by promoting orthogonality among task vectors. DARE~\citep{yu2024language} randomly drops part of the parameters as a dropout-style preprocessing step. MAP~\citep{li2025map} leverages a second-order Taylor expansion and linear regression to estimate the Hessian, providing loss-based guidance for merging. Both MetaGPT~\citep{zhou2024metagpt} and TATR~\citep{sun2025task} explicitly model the loss gap: the former derives a closed-form solution for the merging coefficient, while the latter models the loss gap as the product of gradients and task vectors to quantify knowledge conflicts.

Recent approaches exploit SVD to capture the low-rank structure of task matrices, operating within subspaces to align, orthogonalize, or equalize components so as to preserve shared knowledge while suppressing interference. KnOTS~\citep{stoica2025model}, designed for LoRA fine-tuning, merges adapters by aligning and averaging their right singular vectors. TSV~\citep{gargiulo2025task} applies SVD to each task matrix and whitens the resulting subspaces to reduce inter-task interference. Iso-C~\citep{marczak2025no} replaces the singular values of the summed task matrix with their mean, flattening the singular spectrum to expand the effective subspace and significantly improve the subspace alignment ratio. Iso-CTS~\citep{marczak2025no} extends this by injecting task-specific singular directions, balancing shared and task-specific knowledge. DOGE TA~\citep{wei2025modeling} formulates merging as a constrained optimization problem, constructs a shared subspace, and optimizes a correction vector via projected gradient descent with adaptive task-aware merging coefficients, further narrowing the performance gap to individual task models.

However, all these model merging methods rest on an implicit assumption: that the task-specific and -critical knowledge resides solely within the corresponding task vector and is independent of the pre-trained base model. Consequently, they overlook the LBW parameters within the pre-trained parameters that are critical for each specific task. Our study is the first to examine the role of pre-trained parameters for individual tasks and to protect the pre-trained core subspace by filtering LBW dimensions from task vectors.
\section{Motivation and Observation} \label{sec:motivation}
In this section, we first introduce the general task-vector formulation of model merging and establish the notation used throughout the remainder of the paper. We then validate our motivation from two perspectives across both scalar weight and subspace paradigms: (1) identifying the existence of \textbf{Load-Bearing Wall (LBW) dimensions}, and (2) demonstrating how these components are disrupted during model merging. Together, these observations provide key insights that motivate the design of our method presented in the next section.

\subsection{Preliminary}
Model merging seeks to combine multiple deep neural networks—each fine-tuned from an identical pre-trained model on separate tasks—into a single unified model. Let $\theta_0$ represent the pre-trained weights and $\theta_t$ denote the weights fine-tuned for task $t$, with $t=1,...,T$, where $T$ is the total number of tasks. We write $\theta^{(\ell)}_t$ for the weights at layer $\ell$ corresponding to task $t$, and let $L$ be the total number of layers. The goal of model merging is to determine a merging function $f$ such that the resulting model
\begin{align}\label{ta1}
    \theta^{(\ell)}_M=f(\theta^{(\ell)}_0,\{ \theta^{(\ell)}_t \} _{t=1}^T),\  \forall \  \ell=1,...,L
\end{align}
can perform all tasks that the individual models $\theta_t$ were trained on.

Following \citep{ilharco2022editing}, the layer-wise task vector $\Delta_t^{(\ell)}$ is defined as the the difference between the fine-tuned weights $\theta^{(\ell)}_t$ and the pre-trained weights $\theta^{(\ell)}_0$ at layer $\ell$: $\Delta_t^{(\ell)}=\theta^{(\ell)}_t-\theta^{(\ell)}_0$. For the rest of the paper, the superscript $\ell$ is dropped when the context is clear, and all definitions apply to an arbitrary layer. In Task Arithmetic, the merging function sums all task matrices and adds them back to the pre-trained weights:
\begin{align}\label{ta2}
    \theta_{TA}^{(\ell)}=\theta^{(\ell)}_0+\alpha\Delta_{TA}^{(\ell)},\ with \ \Delta_{TA}^{(\ell)}=\sum_{t=1}^T\Delta_t^{(\ell)}.
\end{align}
where $\alpha$ is a scaling factor chosen on a held-out validation set. This strategy enables the reuse and transfer of knowledge from several fine-tuned models to the pre-trained model without requiring additional training or access to the original training data. Building upon this, subsequent algorithms have focused on studying the conflicts among $\Delta_t$ and their aggregation methods, implicitly assuming that there is no interference between $\theta_0$ and $\Delta_{TA}$ and that they can be directly summed. Our research begins precisely from this premise. Before presenting our empirical evidences, let's first give a brief definition on the LBW dimensions.

\begin{definition}[LBW Dimensions]
    For a task $t$, a parameter dimension $i$ is called an LBW dimension if it contributes significantly to the task performance while receiving only negligible updates during fine-tuning, i.e., $| \Delta_t^{(i)} | \approx 0$. In such dimensions, task-relevant knowledge resides predominantly in the pretrained parameters $\theta_0^{(i)}$ rather than in the task vector $ \Delta_t^{(i)} $.
\end{definition}

\begin{remark}
The notion of LBW dimensions highlights a limitation of purely task-vector-based model merging. Since LBW dimensions receive only negligible updates during fine-tuning, their importance is unlikely to be reflected by the magnitude of the task vector. Consequently, sparsification-based methods, such as TIES-Merging~\citep{yadav2023ties}, may discard these dimensions despite their functional importance to the task. This indicates that part of the task-specific capability is inherited from the pre-trained model and remains anchored in $\theta_0$, rather than being explicitly encoded in $\Delta_t$.
Moreover, LBW dimensions can induce a hidden form of interference during merging. Consider two tasks $(A,B)$ with task vectors $\Delta_A$ and $\Delta_B$. If a dimension $i$ is an LBW dimension for task $A$, then $\Delta_A^{(i)}$ is small while $\theta_0^{(i)}$ remains crucial for preserving the performance of task $A$. When task $B$ introduces a large update $\Delta_B^{(i)}$ on the same dimension, the update may overwrite the pre-trained parameter value required by task $A$, leading to irreversible performance degradation. Importantly, this degradation is not fully explained by conflicts among task vectors themselves; rather, it arises from modifying task-critical pre-trained parameters that are invisible under the task-vector representation. Therefore, LBW dimensions expose a fundamental blind spot of task-vector-based merging methods: they protect or manipulate what is explicitly encoded in task vectors, but may overlook task-relevant knowledge that remains stored in the pre-trained model.
\end{remark}

\subsection{Identifying and Validating LBW Dimensions}
\label{parameter_analysis}

\subsubsection{LBW Dimensions in Scalar Weight} \label{sec:weight_attc}
To validate our hypothesis, we first introduce two metrics for identifying potential LBW dimensions, which are used throughout the subsequent experiments:
\begin{align}\label{fisher}
F_{\theta}=\frac{1}{N}\sum_{i=1}^N\left(\frac{\partial{\mathcal{L}}i}{\partial\theta}\right)^2, \qquad
E_B=\frac{d}{n}\frac{\sum_{\tau\in[\widehat{\Delta_A}]}\tau^2}{\sum_{\tau\in[\Delta_B]}\tau^2}.
\end{align}

Here, $F_{\theta}$ denotes the Fisher information~\citep{kirkpatrick2017overcoming}, where $N$ is the number of samples (1000 in our experiments) and $\mathcal{L}_i$ is the loss of the $i$-th sample. Fisher information quantifies the sensitivity of the loss to a parameter: a larger value indicates that small perturbations to the parameter induce larger changes in the loss. The second metric, $E_B$, is an enrichment score~\citep{subramanian2005gene}, where $\tau$ denotes an individual parameter, $\widehat{\Delta_A}$ is a subset of task vector $\Delta_A$, $n$ is the number of parameters in the subset, and $d$ is the total number of parameters in the layer. Intuitively, $E_B$ measures the extent to which task $B$ concentrates its updates within the dimensions selected from task $A$. Values greater than 1 indicate that task $B$ disproportionately modifies these dimensions relative to the layer-wide average.

\paragraph{Filter Process.} 
For each layer, we first select the $30\%$ of parameters in $\Delta_A$ with the smallest magnitudes. Among them, we retain the $30\%$ with the largest Fisher values, yielding a subset $\widehat{\Delta_A}$ that contains approximately $9\%$ of the layer's parameters. We then compute $E_B$ and retain only dimensions exhibiting statistically significant enrichment ($p<0.05$). The resulting parameters constitute the \underline{\emph{crucial mask}}. This procedure directly targets our hypothesis by identifying dimensions that are weakly represented in the task vector, highly important to task performance, and likely to be modified by another task.

The crucial mask contains fewer than $2\%$ of the model parameters, and all subsequent experiments are conducted on this subset. For comparison, we construct two control masks of identical size. The \underline{\emph{safe mask}} is obtained by selecting the $30\%$ of parameters with the lowest Fisher values and applying the same enrichment filtering procedure. The \underline{\emph{random mask}} is obtained by uniformly sampling parameters at random and applying the same filtering criterion.
\begin{figure*}
    \subfloat{\includegraphics[width = 0.245\textwidth]{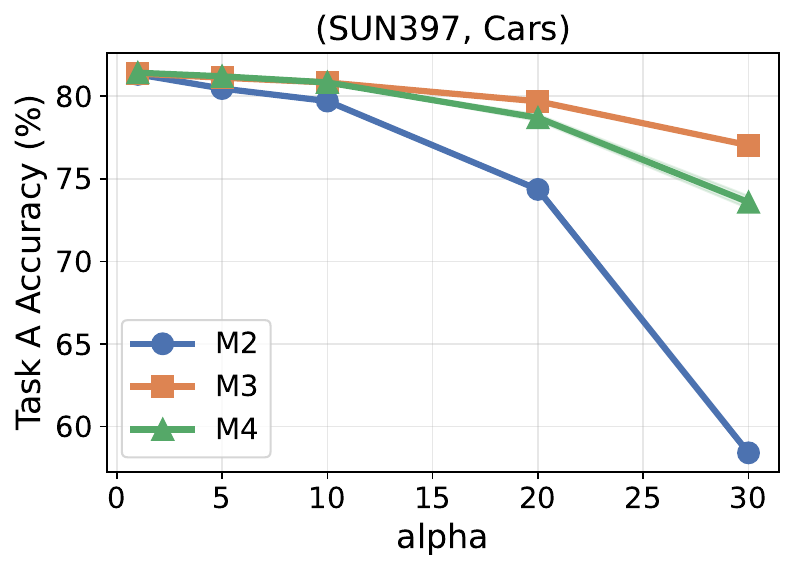}}
    \hfill
    \subfloat{\includegraphics[width = 0.245\textwidth]{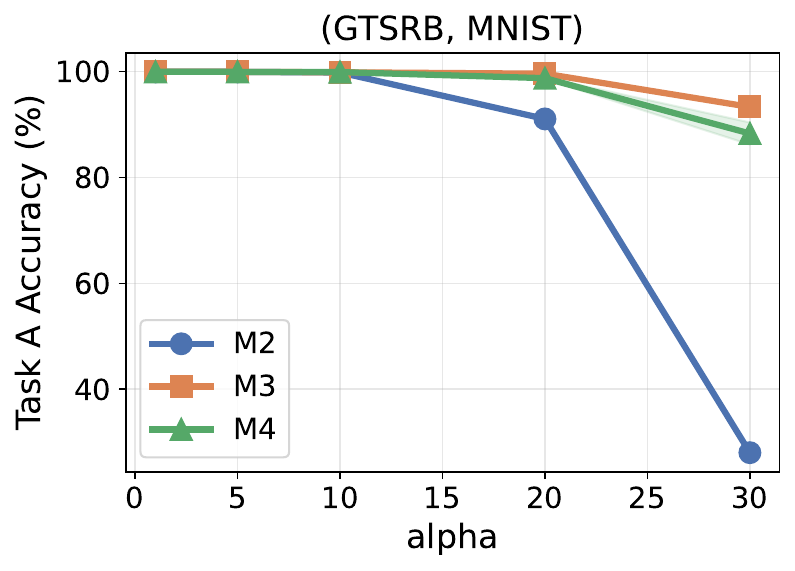}}
    \hfill
    \subfloat{\includegraphics[width = 0.245\textwidth]{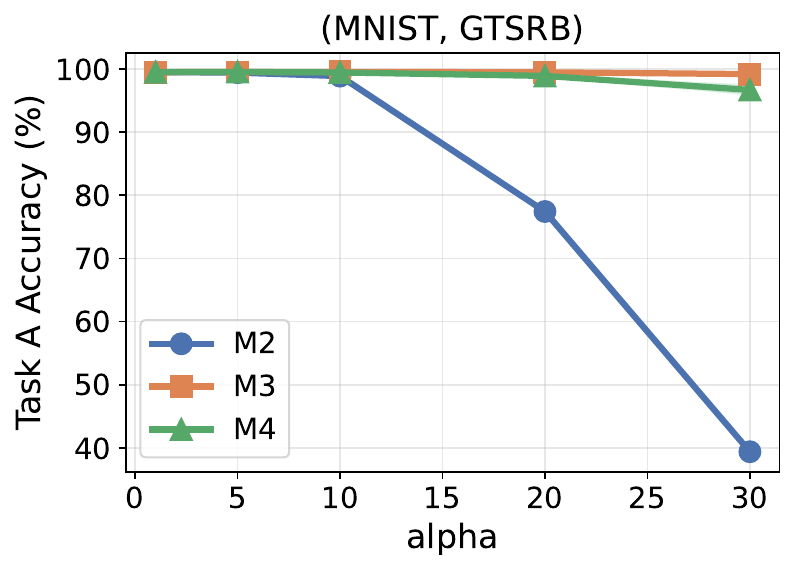}}
    \hfill
    \subfloat{\includegraphics[width = 0.245\textwidth]{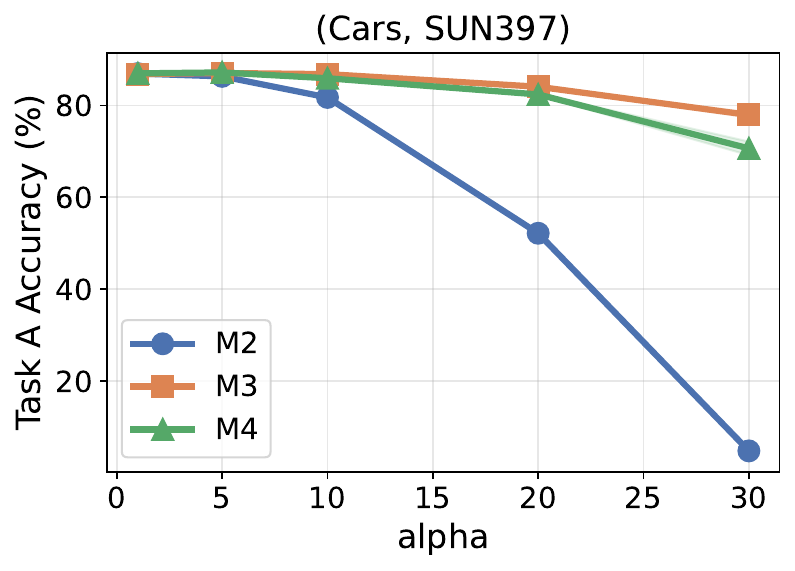}} \\   
    \subfloat{\includegraphics[width = 0.245\textwidth]{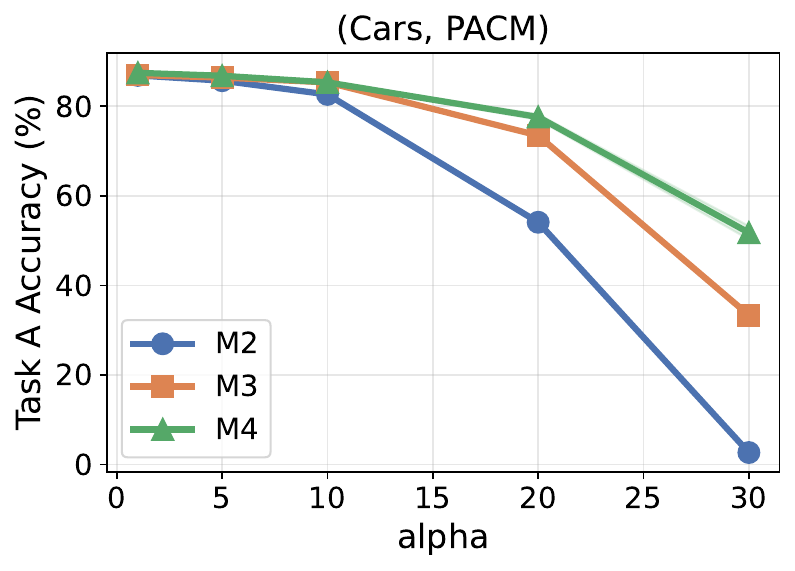}}
    \hfill
    \subfloat{\includegraphics[width = 0.245\textwidth]{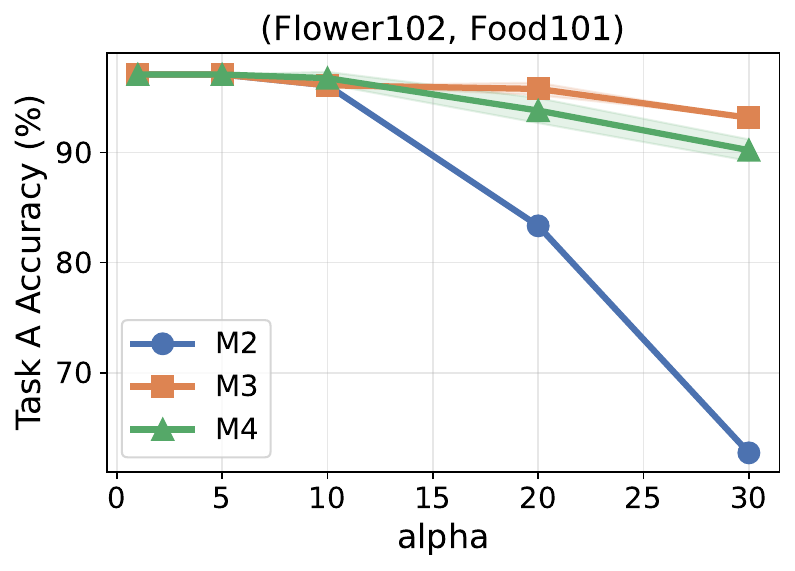}}
    \hfill
    \subfloat{\includegraphics[width = 0.245\textwidth]{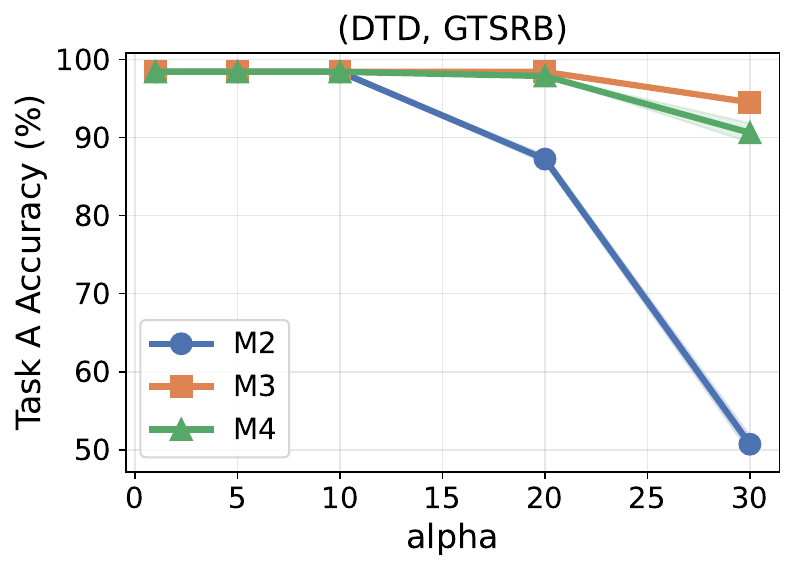}}
    \hfill
    \subfloat{\includegraphics[width = 0.245\textwidth]{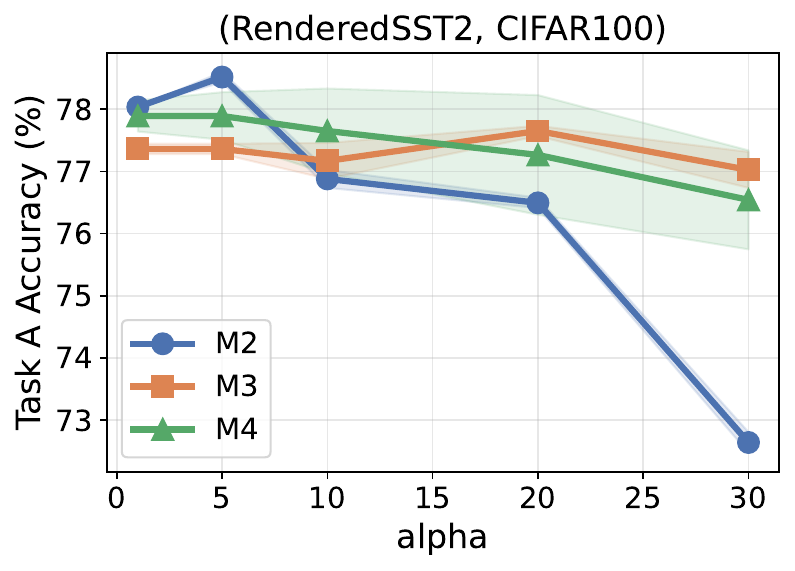}}
    \caption{Accuracy (\%) under varying attack strengths with respective to different masks. Generally, $M2$ is could be easily attacked while $M3$ and $M4$ are rather stable, indicating the importance of parameters within crucial mask. More results please refer to Appendix Sec.~\ref{app:fisher_analysis}.} 
    \label{fig:alpha_attack}
\end{figure*}

\paragraph{Do LBW Parameters Exist?}
We first investigate whether LBW parameters exist. Specifically, we select 16 task pairs $(A,B)$ from commonly employed benchmarks~\citep{wang2024localizing} and compute the crucial mask for task $A$. We then use task $B$'s task vector to perturb the corresponding dimensions. Let $M0$ denote the performance of task $A$'s fine-tuned model. We first restore the crucial-mask parameters to their pre-trained values, producing $M1$, to verify whether these dimensions indeed derive their functionality from the pre-trained model. We then perform three controlled perturbation experiments on the crucial, safe, and random masks, respectively: $\theta_0+\alpha\tau_B$, where $\alpha\in\{1,5,10,20,30\}$ controls the perturbation strength and $\tau_B$ denotes parameters from the corresponding dimensions in task vector $\Delta_B$. The resulting performances are denoted by $M2$ (crucial mask), $M3$ (safe mask), and $M4$ (random mask). All experiments are conducted using ViT-B/16.

Representative results are shown in Figure~\ref{fig:alpha_attack} and Table~\ref{tab:attk1}, with additional results provided in the Appendix. The near-identical performance of $M0$ and $M1$ indicates that the selected dimensions contribute little to the task vector itself, confirming that their functionality primarily originates from the pre-trained model. In contrast, perturbing the crucial mask ($M2$) causes a dramatic performance drop, whereas perturbations of identical magnitude applied to the safe mask ($M3$) and random mask ($M4$) lead to substantially smaller degradation. As shown in Figure~\ref{fig:alpha_attack}, performance decreases monotonically with increasing perturbation strength for all three masks, with the most severe degradation consistently observed on the crucial mask. These results provide direct evidence for the existence of LBW parameters.

\paragraph{Are LBW Parameters Damaged During Merging?}

Having established the existence of LBW parameters, we next examine whether these parameters are disrupted during model merging. Specifically, we first construct the merged model $\theta_0+\Delta_A+\Delta_B$ and evaluate its performance on task $A$. We then perform a targeted restoration by resetting only the crucial-mask parameters to their original pre-trained values in $\theta_0$, while keeping all other parameters exactly the same as in the merged model. The restored model is then evaluated again on task $A$. If this restoration leads to a clear performance recovery, it indicates that the crucial parameters of task $A$ have been overwritten during merging, and that such disruption directly contributes to the degradation of task $A$.

Representative results are reported in Table~\ref{tab:attk2}, with additional results provided in the Appendix Sec.~\ref{app:fisher_analysis}. Restoring only the crucial-mask parameters yields substantial performance improvements on task $A$. Given that the crucial mask accounts for less than $2\%$ of the total parameters, the magnitude of the recovery is remarkably large. These findings indicate that existing task-vector-based merging methods indeed modify LBW dimensions and thereby impair task-specific knowledge embedded in the pre-trained model. 
\begin{table}[t]
    \centering
    \scriptsize
    \setlength\tabcolsep{2pt}
    \begin{minipage}[t]{0.48\textwidth}
    \vspace{-0.05in}
      \resizebox{\linewidth}{!}{ 
      \begin{tabular}{c|cc>{\columncolor{pink!30}}ccc}
      \hline 
      \toprule
      ($A$,$B$) & $M0$ & $M1$ & $M2$ & $M3$ & $M4$ \\
      \midrule
      (Cars, SUN397) & 86.49 & 86.61 & 5.04 & 77.93 & 70.64 \\

      (GTSRB, MNIST) & 99.92 & 99.92 & 28.04 & 93.39 & 90.47 \\

      (Food101, Flowers102) & 89.56 & 89.60 & 55.29 & 84.01 & 83.16 \\

      (DTD, GTSRB) & 98.40 & 98.40 & 50.71 & 94.50 & 90.60 \\

      (OxfordIIITPet, SUN397) & 94.84 & 94.84 & 51.54 & 93.03 & 91.94 \\

      \bottomrule
      \hline 
      \end{tabular}
      }
      \caption{Accuracy (\%) under different parameters modifications. The value of $\alpha$ is set to 30.0.}
      \vspace{0.02in}
      \label{tab:attk1}
    \end{minipage}
    ~~~~
    \begin{minipage}[t]{0.48\textwidth}
    \vspace{-0.05in}
      \resizebox{\linewidth}{!}{
      \begin{tabular}{c|cc|>{\columncolor{pink!30}}c}
      \hline 
      \toprule
      ($A$, $B$) & $\theta_0+\Delta_A+\Delta_B$ & $\theta_0$ on crucial  & $\Delta_{\rm acc}$ \\
      \midrule
      (SUN397, Cars) & 70.97 & 71.78  & +0.81 \\
      (Cars, SUN397) & 78.62 & 79.36  & +0.74 \\

      (GTSRB, MNIST) & 91.55 & 93.47 & +1.91 \\

      (Cars, PACM) & 74.69 & 77.03 & +2.33 \\

      (DTD, GTSRB) & 92.55 & 94.68 & +2.13 \\

      \bottomrule
      \hline 
      \end{tabular}
      }
      \caption{$\Delta_{acc}$ represents the accuracy (\%) change due to the crucial mask replacement.}
      \vspace{0.02in}
      \label{tab:attk2}
    \end{minipage}
\end{table}

\subsubsection{LBW Dimensions in Subspace} 
\label{sec:subspace_analysis}
We have established the existence of LBW dimensions and their vulnerability during model merging. However, whether a similar phenomenon exists in subspace-based model merging methods~\citep{yang2026model,ruan2025task} remains unclear. While the Fisher information analysis reveals the vulnerability of specific parameters at a microscopic level, subspace-based approaches operate on low-rank geometric structures rather than individual weights. Therefore, to understand how task interference arises in these methods, we extend our analysis from parameter space to subspace geometry.

We perform SVD on both the pre-trained weight matrix $\theta_0 \in \mathbb{R}^{m \times n}$ and the task vectors $\Delta_1, \dots, \Delta_T \in \mathbb{R}^{m \times n}$:
\begin{align}\label{svd}
\theta_0 = U_0 \Sigma_0 V_0^\top, \quad \text{and} \quad \Delta_t = U_t \Sigma_t V_t^\top, \quad \forall t=1,\dots,T,
\end{align}
where $\Sigma_{*}$ denotes the diagonal matrix of singular values, and $U$ and $V$ denote the left and right singular vectors, respectively. Physically, for any input feature $x \in \mathbb{R}^n$, the forward pass transformation is governed by $\theta_0 x = U_0 \Sigma_0 V_0^{\top} x$. In this formulation, the right singular vectors $V_0$ act as an orthonormal basis for the input space, where the projection coordinates $V_0^{\top} x$ are scaled by the singular values $\Sigma_0$. Consequently, the top-$K$ right singular vectors $V_{0,K}$ represent the principal directions that preserve the maximum variance of the representations, thereby capturing the dominant structural foundation of the pre-trained model.

To identify LBW dimensions in the spectral domain, we extract the top-$K$ singular vectors of the pre-trained weights to form the base matrix $V_{0,K} \in \mathbb{R}^{n \times K}$. Similarly, we extract the top-$k$ singular vectors of $\Delta_t$ to form the active task matrix $V_{t,k} \in \mathbb{R}^{n \times k}$. Unless otherwise specified, we set $(K,k)=(15,8)$ throughout all experiments.

We define the \textit{LBW subspace} $V_{t,LBW}$ of task $t$ as the component of the pre-trained core subspace $V_{0,K}$ that lies outside the active update subspace of task $t$. Formally, it is obtained by projecting $V_{0,K}$ onto the orthogonal complement of $V_{t,k}$ followed by orthonormalization:
\begin{align}\label{pact1}
V_{t,LBW} = \text{orth}\left((I - V_{t,k}V_{t,k}^{\top})V_{0,K}\right)
\in \mathbb{R}^{n \times d_{pact}}.
\end{align}
Specifically, the projection matrix $I - V_{t,k}V_{t,k}^{\top}$ filters out components of the pre-trained core $V_{0,K}$ that lie within the active task updates $V_{t,k}$. This projection successfully isolates the pre-trained directions that are critical for task performance yet receive negligible updates during fine-tuning. The $\text{orth}$ operator then orthonormalizes these remaining components to construct a stable, non-redundant basis $V_{t,LBW}$ with dimension $d_{pact}$, which is at most $K$. Conceptually, $V_{t,LBW}$ serves as the spectral analogue of the LBW dimensions, capturing pre-trained feature directions that remain untouched during task adaptation.

\paragraph{Causal Validation via Subspace Ablation.}
To investigate whether model performance depends on $V_{t,LBW}$, we conduct a global subspace ablation study. Given a task-specific model $\theta_t = \theta_0 + \Delta_t$, we remove the LBW subspace by constructing $\theta_{abla} = \theta_t(I - V_{t,LBW}V_{t,LBW}^{\top}).$
We reasonably assume $\Delta_t$ is approximately orthogonal to $V_{t,LBW}$ by construction (i.e., $\Delta_tV_{t,LBW}\approx0$), this operation primarily removes the pre-trained components contained in $V_{t,LBW}$ while preserving the task-specific updates:
\begin{align}\label{ablation_approx} 
\theta_{abla} \approx \theta_0 (I - V_{t,LBW} V_{t,LBW}^\top) + \Delta_t 
\end{align}
We provide a formal derivation in Appendix Sec.~\ref{appendix:proof}. As a control group, we ablate a random orthogonal subspace $V_{rand}$ of identical dimensionality strictly within the non-core pre-trained space ($V_{0,K}^\perp$): $\theta_{rand} = \theta_t (I - V_{rand} V_{rand}^\top)$. To ensure a fair comparison, the control subspace $V_{rand}$ is constructed under two constraints: it must share the identical dimension $d_{pact}$ and lie strictly within the non-core pre-trained space orthogonal to $V_{0,K}$. Formally, we project a random Gaussian matrix $X_{rand}$ onto the orthogonal complement of the pre-trained core to obtain $(I - V_{0,K}V_{0,K}^{\top})X_{rand}$, and then apply QR decomposition to extract the orthonormal basis. From this basis, we randomly select and re-orthonormalize $d_{pact}$ columns to yield $V_{rand}$. Unlike $V_{t,LBW}$ which preserves the structured, task-essential components derived from the pre-trained core $V_{0,K}$, $V_{rand}$ represents purely unstructured directions within the non-core space $V_{0,K}^{\perp}$.

\begin{wraptable}[12]{r}{0.59\textwidth}
\vspace*{-\baselineskip}
\centering
\scriptsize
\setlength\tabcolsep{2pt}
\begin{tabular}{c|ccc>{\columncolor{pink!30}}c}
\hline
\toprule
Task $t$& $\theta_t$ & $\theta_{abla}\ (\Delta_{\rm acc}^{LBW})$ & $\theta_{rand}\ (\Delta_{\rm acc}^{rand})$  & $\Delta_{\rm acc}^{LBW}/\Delta_{\rm acc}^{rand}$ \\
\midrule
EMNIST & 99.78 & 17.00 (-82.78) & 99.75 (-0.03) & 3104.25$\times$\\
FashionMNIST & 95.28 & 9.78 (-85.50) & 95.14 (-0.14) & 610.71$\times$\\
GTSRB & 99.92 & 1.91 (-98.01) & 99.86 (-0.06) & 1566.60$\times$\\
KMNIST & 99.94 & 11.92 (-88.02) & 99.88 (-0.06) & 1467.00$\times$\\
STL10 & 99.60 & 11.00 (-88.60) & 99.33 (-0.27) & 332.25$\times$\\
SVHN & 96.76 & 7.40 (-89.36) & 96.51 (-0.25) & 352.74$\times$\\
\bottomrule
\hline
\end{tabular}
\caption{$\Delta_{acc}^{*}$ represents the accuracy (\%) change due to the subspace filtering. The random subspace filtering is averaged across three random seeds. The numbers in parentheses indicate the decrease in accuracy compared to the $\theta_t$.}
\label{tab:attk3}
\end{wraptable}

Table~\ref{tab:attk3} reports results on 19 tasks using ViT-B/16. Although $V_{t,LBW}$ accounts for only \textbf{1.17\%} of the total dimensions (780 out of 66,816), its removal leads to severe performance degradation across tasks. For example, EMNIST accuracy drops from \textbf{99.78\%} to \textbf{17.00\%} ($\downarrow 82.78\%$), while GTSRB accuracy decreases from \textbf{99.92\%} to \textbf{1.91\%} ($\downarrow 98.01\%$). In contrast, ablating a random control subspace has a negligible effect, with EMNIST and GTSRB retaining \textbf{99.75\%} ($\downarrow 0.03\%$) and \textbf{99.86\%} ($\downarrow 0.06\%$) accuracy, respectively. Additional results are provided in Appendix Sec.~\ref{app:geometric_subspace}.

On average, the performance degradation caused by removing $V_{t,LBW}$ is \textbf{458.05$\times$ larger} than that observed for the random control subspace. These results provide strong empirical evidence that the pre-trained components contained in $V_{t,LBW}$ play a disproportionately important role in downstream task performance, supporting our interpretation of $V_{t,LBW}$ as the subspace-level manifestation of LBW dimensions.

\paragraph{Interference of LBW Subspaces During Model Merging.}
Having established $V_{t,LBW}$ as the LBW subspace, we analyze how standard model merging disrupts these structures by introducing two geometric metrics for any task pair $(A, B)$ to quantify this interference. First, we define the symmetric \textit{Sacred Space Similarity} ($Sim$) to quantify the geometric overlap between the LBW subspaces of two distinct tasks, and the asymmetric \textit{Hidden Interference Ratio} ($Inf$) to quantify the direct intrusion of one task's updates into another task's parameters:
\begin{align}\label{sim_metric}
    Sim(A, B) = \frac{\|(V_{A,LBW})^\top V_{B,LBW}\|_F^2}{\min(r_A, r_B)}, \qquad
Inf(B \to A) = \frac{\|(V_{B,k})^\top V_{A,LBW}\|_F^2}{k}
\end{align}
where $r_A$ and $r_B$ denote the rank of their respective LBW subspace.Mathematically, $Sim(A,B)$ computes the sum of squared cosines of principal angles between two subspaces, a classical metric heavily utilized in Grassmanian manifold alignment~\citep{fernando2013unsupervised} and task singular vector synchronization~\citep{marczak2025no}. $Inf$ computes the exact proportion of Task $B$'s explicit active task vector updates ($V_{B,k}$) that invasively project onto the LBW subspace of Task $A$. Unlike parameter-wise ratios, the $Inf$ metric is purely geometric and scale-invariant, strictly evaluating the directional collisions between Task $B$'s core task vector updates and Task $A$'s implicit dependency on its LBW subspace. This projection-based interference measurement is closely related to the gradient projection metrics developed in continual learning for preventing catastrophic forgetting~\citep{farajtabar2020orthogonal,saha2021gradient}.
\begin{wrapfigure}[13]{r}{0.6\textwidth}
    \vspace*{-1.\baselineskip}
    \subfloat{\includegraphics[width = 0.29\textwidth]{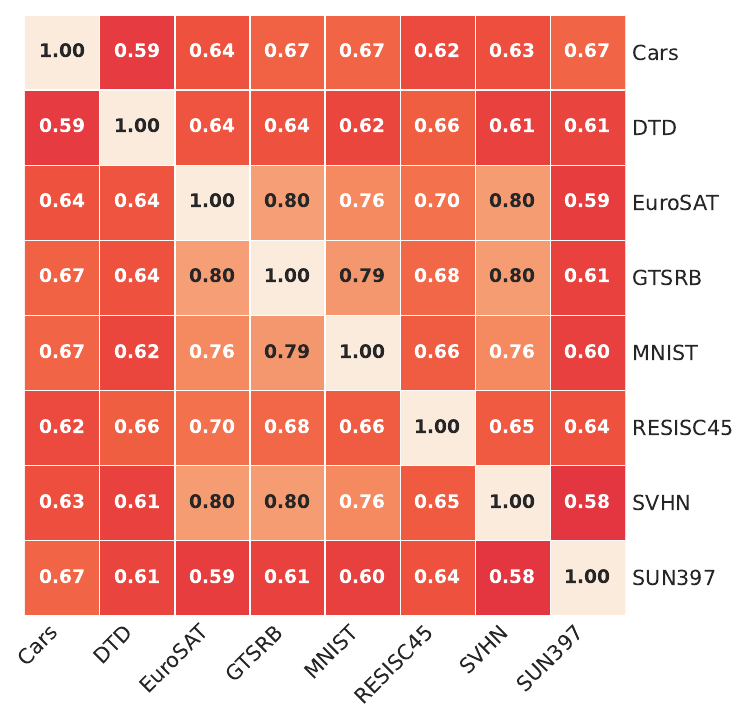}}
    \hfill
    \subfloat{\includegraphics[width = 0.29\textwidth]{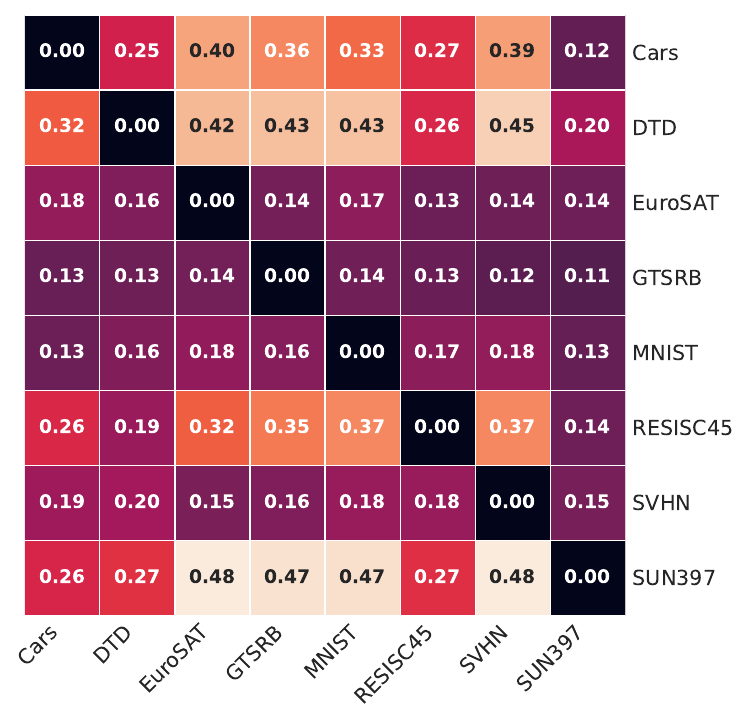}}
    \caption{$Sim$ and $Inf$ of the \texttt{mlp.c\_fc} layer in Block 0.} 
    \label{fig:interfb00}
\end{wrapfigure}
The empirical severity of these subspace-level collisions is fully exposed in the heatmaps for the foundational Block 0 (Figure~\ref{fig:interfb00}). On one hand, the $Sim$ heatmap reveals that \textbf{the similarity between the LBW subspaces of different tasks is relatively low to some extent}, with values dropping as low as $0.59$. This low similarity confirms that different tasks depend on highly heterogeneous subsets of pre-trained parameters, proving that a universal, task-agnostic ``safe zone'' does not exist; instead, each task carves out a unique, indispensable dependency on its own LBW subspace. On the other hand, the $Inf$ heatmap uncovers a non-negligible geometric collision, demonstrating that \textbf{the active updates of task vectors blindly invade the LBW subspaces of other tasks}. While the diagonal remains strictly $0.00$—confirming self-orthogonality—the off-diagonal entries skyrocket, showing that the explicit task vector updates of one task can obliterate up to \textbf{48\%} of the LBW subspace implicitly relied upon by another (for instance, SUN397 severely intrudes upon the LBW subspace of EuroSAT). This profound geometric collision highlights that naive model merging severely corrupts the foundational LBW subspace of downstream tasks (refer to Appendix for more results). This directly mandates a layer-wise mechanism to explicitly shield these LBW subspace during merging, motivating our design.

\section{Preserving Anchored Cores in Task-vectors}

While Fisher-based methods can theoretically locate these vital parameters, their heavy computational cost and requirement for training data preclude the possibility of explicitly protecting them in practical, data-free model merging scenarios. We therefore need a merging framework that is task-agnostic, data-free, and capable of geometrically shielding these implicit dependencies layer by layer. To this end, we propose \texttt{PACT}. Instead of directly summing task matrices, PACT operationalizes the isolated ``Sacred Spaces'' ($V_{t,LBW}$ derived in Eqn.~\ref{pact1}) into a universal geometric shield, enforcing strict orthogonal filtering to neutralize hidden interference before integration.

\begin{wrapfigure}{r}{0.50\textwidth}
    \begin{minipage}{0.50\textwidth}
        \vspace{-0.3in}
        \begin{algorithm}[H]
            \caption{The Process of \texttt{PACT}}
            \label{alg:pact}
            \begin{algorithmic}[1]
                \STATE \textbf{Input:} Pre-trained model $\theta_0$, task vectors $\{\Delta_t\}_{t=1}^{T}$
                \STATE \textbf{Output:} Filtered task vectors $\{\tilde{\Delta}_t\}_{t=1}^{T}$
                \STATE Compute the principal subspace $V_{0,K}$ of $\theta_0$.
                \FOR{$t=1,\ldots,T$} 
                    \STATE Compute the task subspace $V_{t,k}$.
                    \STATE Compute the anchored core subspace $V_{t,LBW}$ using Eqn.~\ref{pact1}.
                \ENDFOR

                \FOR{$t=1,\ldots,T$} 
                    \STATE Construct the protected subspace $V_{\mathrm{protect}}^{t}$ using Eqn.~\ref{pact2}.
                    \STATE Filter the task vector using Eqn.~\ref{pact3}.
                \ENDFOR
            \end{algorithmic}
        \end{algorithm}
        \vspace{-0.35in}
    \end{minipage}
\end{wrapfigure}

As defined in Section~\ref{sec:motivation}, $V_{t,LBW}$ precisely identifies the orthogonal anchor directions that a single task $t$ implicitly relies upon. However, merging $T$ distinct tasks requires a holistic defense mechanism: we must ensure that no task's explicit update interferes with the load-bearing walls of any other task. Simply filtering against each $V_{t,LBW}$ sequentially is mathematically flawed and computationally inefficient, as the sacred spaces across different tasks often share significant overlaps (reflected by the high similarities in Figure~\ref{fig:interfb00}). To construct a unified, non-redundant protection boundary for a given task t, we aggregate the sacred spaces of all other tasks. We then extract a unified orthonormal column basis using QR decomposition, denoted by the operator ${\rm orth}(\cdot)$:
\begin{align}\label{pact2}
    V_{\rm protect}^t = {\rm orth} \Big( \big[ V_{1, LBW} \ \big| \ \dots \ \big| \ V_{t-1, LBW} \ \big| \ V_{t+1, LBW} \ \big| \ \dots \ \big| \ V_{T, LBW} \big] \Big)
\end{align}

This resulting matrix $V_{\mathrm{protect}}^{t}$ serves as the \emph{Global Orthogonal Shield} for task $t$. It elegantly captures the union of all critical pre-trained dependencies required by the rest of the tasks, while eliminating overlapping basis vectors to prevent over-constraining the update space. Equipped with the global shield $V_{\mathrm{protect}}^{t}$, we can now surgically sanitize the update matrix of task $t$. We enforce strict non-interference by subtracting the projection of $\Delta_t$ onto the protected subspace:
\begin{align}\label{pact3}
    \tilde\Delta_t=\Delta_t-\Delta_t V_{\mathrm{protect}}^{t}(V_{\mathrm{protect}}^{t})^\top
\end{align}
where $V_{\mathrm{protect}}^{t}(V_{\mathrm{protect}}^{t})^\top$ is the orthogonal projection operator onto the protection space. Applying it to $\Delta_j$ extracts the components of $\Delta_t$ that would overwrite the anchor directions of other tasks. Subtracting these components removes such detrimental parts, retaining only the portions of $\Delta_t$ that are harmless to other tasks, thereby completing the filtering. In this process, the filtering of each task is jointly determined by the anchor directions of all other tasks, embodying a mutual protection mechanism across tasks.

The complete computational procedure is presented in Algorithm~\ref{alg:pact}. Considering that our algorithm requires performing $T+1$ SVD decompositions for each layer of the model, its time complexity reaches $O(mn\cdot \min(m,n))$. Moreover, it requires storing the full $U\in \mathbb{R}^{m\times m}$ and $V\in \mathbb{R}^{n\times n}$, which imposes substantial memory pressure for large-scale matrices. To address this, we propose an efficient alternative of \texttt{PACT}. Specifically, we can replace the full SVD decomposition in Eqn.~\ref{svd} with randomized SVD~\citep{halko2011finding}. Randomized SVD leverages random projections for low-rank approximation, and its ability to efficiently and approximately obtain the top-$k$ singular values and singular vectors is suited to the top-$k$ extraction process of \texttt{PACT}. Meanwhile, the steep singular value spectrum ensures the efficiency of randomized SVD. Through the adaptation of randomized SVD, we can reduce the time complexity to $O(mnk)$, while only needing to store the narrow matrix $Q\in \mathbb{R}^{m\times (k+p)}$ and a few small matrices, thereby substantially alleviating memory pressure as well. The computational complexity analysis is provided in the Appendix Sec.~\ref{sec:complexity}.

\paragraph{Theoretical Analysis.}
\label{sec:theoretical_analysis}

To rigorously justify the empirical success of \texttt{PACT}, we provide a theoretical analysis demonstrating its preservation of individual task performance.

\begin{theorem}[Projected First-Order Performance Preservation] \label{thm:projected_performance_preservation} 
Let $\theta_j=\theta_0+\Delta_j$ denote the task-specific fine-tuned parameters for task $j$, and let $P_j=V_{\mathrm{protect}}^j$ be the orthonormal basis of the protection space used by \texttt{PACT}. The filtered update is defined as $\tilde{\Delta}_j = \Delta_j \left( I-P_jP_j^\top \right).$ 
Assume that $\mathcal{L}_j$ is locally $\beta_j$-smooth around $\theta_j$. Then the degradation caused by the projection filtering satisfies 
$$
\mathcal{L}_j(\theta_0+\tilde{\Delta}_j) - \mathcal{L}_j(\theta_0+\Delta_j) \leq \left\| \nabla \mathcal{L}_j(\theta_j) P_j \right\|_F \left\| \Delta_j P_j \right\|_F + \frac{\beta_j}{2} \left\| \Delta_j P_j \right\|_F^2 . 
$$
In particular, the first-order degradation is bounded by 
\begin{equation} 
\Delta \mathcal{L}_j^{(1)} \leq \left\| \nabla \mathcal{L}_j(\theta_j) P_j \right\|_F \left\| \Delta_j P_j \right\|_F . \end{equation} 
\end{theorem} 

Theorem~\ref{thm:projected_performance_preservation} shows that the self-performance degradation is not determined by the overall norm of the update, but by the discarded update energy and the task sensitivity inside the protection space. Equivalently, defining 
$$
\rho_{g,j} = \frac{ \left\| \nabla \mathcal{L}_j(\theta_j) P_j \right\|_F }{ \left\| \nabla \mathcal{L}_j(\theta_j) \right\|_F }, \qquad \rho_{\Delta,j} = \frac{ \left\| \Delta_j P_j \right\|_F }{ \left\| \Delta_j \right\|_F }, 
$$ 
the first-order term is controlled by 
\begin{equation} 
\Delta \mathcal{L}_j^{(1)} \leq \left\| \nabla \mathcal{L}_j(\theta_j) \right\|_F \left\| \Delta_j \right\|_F \rho_{g,j}\rho_{\Delta,j}. \end{equation} 

\begin{corollary}[Active-Subspace Relaxation] \label{cor:active_subspace_relaxation} 
Let $V_{j,k}$ be the orthonormal basis of the active subspace of task $j$, and define $\cos(\phi_j) = \left\| V_{j,k}^{\top}P_j \right\|_2 . $
Suppose that both the gradient and the task update are approximately concentrated in $V_{j,k}$, namely 
$$
\left\| \nabla \mathcal{L}_j(\theta_j) \left( I-V_{j,k}V_{j,k}^{\top} \right) \right\|_F \leq \epsilon_g \left\| \nabla \mathcal{L}_j(\theta_j) \right\|_F, 
$$ and 
$$
\left\| \Delta_j \left( I-V_{j,k}V_{j,k}^{\top} \right) \right\|_F \leq \epsilon_\Delta \left\| \Delta_j \right\|_F . 
$$
Then the degradation caused by \texttt{PACT}'s filtering satisfies 
\begin{equation} 
\begin{aligned}
\mathcal{L}_j(\theta_0+\tilde{\Delta}_j) - \mathcal{L}_j(\theta_0+\Delta_j) \leq & \left\| \nabla \mathcal{L}_j(\theta_j) \right\|_F \left\| \Delta_j \right\|_F \left( \cos(\phi_j)+\epsilon_g \right) \left( \cos(\phi_j)+\epsilon_\Delta \right) \\ & + \frac{\beta_j}{2} \left\| \Delta_j \right\|_F^2 \left( \cos(\phi_j)+\epsilon_\Delta \right)^2 . \nonumber
\end{aligned} 
\label{eq:angle_relaxed_bound} 
\end{equation} 
In the exact active-subspace case where $\epsilon_g=\epsilon_\Delta=0$, the first-order degradation reduces to 
\begin{equation} 
\Delta \mathcal{L}_j^{(1)} \leq \left\| \nabla \mathcal{L}_j(\theta_j) \right\|_F \left\| \Delta_j \right\|_F \cos^2(\phi_j). \end{equation} 
\end{corollary}
\section{Experiments}
In this section, we evaluate the model merging performance of \texttt{PACT} under various experimental settings. We examine the effectiveness of \texttt{PACT} when combined with representative model merging approaches, e.g., TA, TSV-M and Iso‑C. Experiments are first conducted on fully fine‑tuned vision models, followed by tests on LoRA fine‑tuned vision models. We also investigate the effect of replacing full SVD with randomized SVD. More details and experiments (e.g., hyper-parameter sensitivity analysis) please refer to Appendix Sec.~\ref{app:ablation_studies} and Sec.~\ref{sec:exp_detail}, respectively.

\subsection{Fully Fine-tuned Vision Models}

\begin{table}[htbp]
\centering
\scriptsize   
\setlength{\tabcolsep}{4pt}
\begin{tabular}{lccccccccc}
\toprule
\multirow{2}{*}{Method} & \multicolumn{3}{c}{ViT-B/32} & \multicolumn{3}{c}{ViT-B/16} & \multicolumn{3}{c}{ViT-L/14} \\
\cmidrule(lr){2-4} \cmidrule(lr){5-7} \cmidrule(lr){8-10}
 & 8 tasks & 14 tasks & 20 tasks & 8 tasks & 14 tasks & 20 tasks & 8 tasks & 14 tasks & 20 tasks \\
\midrule
Zero-shot    & 48.3   & 57.2   & 56.1   & 55.3   & 61.3   & 59.7   & 64.7   & 68.2   & 65.2   \\
Fine-tuned   & 92.8   & 90.9   & 91.3   & 94.6   & 92.8   & 93.2   & 95.8   & 94.3   & 94.7   \\
Weight Averaging & 66.3(72.1) & 64.3(71.1) & 61.0(67.5) & 72.2(76.6) & 69.5(74.8) & 65.3(70.4) & 79.6(83.2) & 76.7(81.1) & 71.6(75.6) \\
TIES          & 75.1(81.0) & 68.0(74.8) & 63.4(69.9) & 79.7(84.3) & 73.2(78.7) & 68.2(73.3) & 86.9(90.7) & 79.5(84.1) & 75.7(79.8) \\
Consensus TA  & 75.0(80.8) & 70.4(77.4) & 65.4(72.0) & 79.4(83.9) & 74.4(79.9) & 69.8(74.9) & 86.3(90.1) & 82.2(86.9) & 79.0(83.2) \\
Iso-CTS & 86.2(92.8) & \underline{81.7(89.7)} & 78.1(85.5) & 91.1(96.1) & \textbf{86.4(92.8)} & 82.4(88.4) & \underline{94.7(98.8)} & \textbf{91.0(96.3)} & \underline{90.1(94.9)} \\
\midrule
Task Arithmetic   & 70.8(76.5) & 65.3(72.1) & 60.5(66.8) & 75.4(79.6) & 70.5(75.9) & 65.8(70.8) & 84.9(88.7) & 79.4(84.0) & 74.0(78.1) \\
\rowcolor{pink!30}\texttt{PACT-TA}   & 75.2(81.0) & 66.4(73.3) & 59.9(66.4) & 85.5(90.3) & 76.8(82.7) & 68.7(74.0) & 92.8(96.9) & 87.0(92.0) & 83.0(87.7) \\
TSV-M         & 85.9(92.3) & 80.1(87.9) & 77.1(84.3) & 89.0(93.9) & 84.6(91.0) & 80.6(86.5) & 93.0(97.0) & 89.2(94.4) & 87.7(92.5) \\
\rowcolor{pink!30}\texttt{PACT-TSV-M}   & \textbf{88.1(94.8)} & \textbf{82.0(89.6)} & \textbf{80.3(87.6)} & \underline{91.2(96.3)} & 85.4(91.5) & \textbf{83.8(89.6)} & 94.4(98.5) & 89.5(94.4) & 89.7(94.3) \\
Iso-C   & 86.3(92.9) & 80.3(88.1) & 75.5(82.5) & 90.6(95.6) & 84.8(91.1) & 79.6(85.4) & 94.2(98.3) & 89.3(94.5) & 87.6(92.2) \\
\rowcolor{pink!30}\texttt{PACT-Iso-C}   & \underline{87.3(94.1)} & 81.3(89.0) & \underline{78.7(86.0)} & \textbf{92.0(97.2)} & \underline{86.1(92.3)} & \underline{83.4(89.3)} & \textbf{95.1(98.3)} & \underline{90.5(95.5)} & \textbf{90.3(95.0)} \\
\bottomrule
\end{tabular}
\caption{\texttt{PACT-Iso-C} achieves SOTA performance for all backbones on all evaluated scenarios. We present average absolute accuracy and average normalized accuracy (in bracket) in $\%$. The best method in \textbf{bold} and the second-best \underline{underlined}.}
\label{tab:full-vit}
\end{table}

We evaluate our approaches over sets of 8, 14, and 20 datasets, following \citet{marczak2025no}. We provide the details of the datasets in Appendix. We consider three variants of CLIP \citep{radford2021learning} with ViT-B/32,ViT-B/16 and ViT-L/14 as visual encoders \citep{dosovitskiy2020image}. We use the checkpoints fine-tuned on the tasks above, provided in \citet{wang2024localizing}. The hyperparameters (K, k) are set to be (15,8).

We compare our approaches with the following model merging methods: weight averaging~\citep{wortsman2022model}, Task Arithmetic~\citep{ilharco2022editing}, TIES-Merging~\citep{yadav2023ties}, Consensus TA~\citep{wang2024localizing}, TSV-M~\citep{gargiulo2025task}, Iso-C and Iso-CTS~\citep{marczak2025no}. We include the results of the zero-shot model and fine-tuned models serving as lower-and upper-bound, respectively. We compare the results based on absolute and normalized accuracy following standard practice~\citep{wang2024localizing,gargiulo2025task,marczak2025no}.

Table \ref{tab:full-vit} presents the performance of multi-task model merging. \texttt{PACT} consistently improves across all settings over all three base methods: TA, TSV-M, and Iso-C. Notably, \texttt{PACT-IsoC} achieves SOTA results in all 9 settings, and its gain over Iso-C becomes even more pronounced as the number of tasks increases, surpassing the protection offered by Iso-CTS in the common space. \texttt{PACT-TSV-M} also attains SOTA performance on ViT-B/32. Moreover, \texttt{PACT-TA} delivers highly competitive results, with the three experiments on ViT-L/14 all approaching the SOTA level. These substantial improvements further corroborate the existence of the LBW effect within pre-trained parameters.

\subsection{LORA-adaptive Vision Models}

To evaluate our approaches in low-rank adaptation scenario, we follow the evaluation protocol of KnOTS~\citep{stoica2025model}, a recent SOTA method for merging LoRA fine-tuned models. We use codebase and checkpoints provided by KnOTS: ViT-B/32 and ViT-L/14 fine-tuned with rank 16 LoRA~\citep{hu2022lora} on 8 vision tasks. To adapt our methodology to the low-rank regime, rather than operating on the reconstructed dense matrix $\Delta W_t=B_tA_t$, we simply perform SVD on the small matrix $A_t$. This is mathematically justified as they share the exact same row space, providing a computationally free trick to extract the task subspace. We compare our approaches with Iso-C, Iso-CTS, TIES and DARE-TIES~\citep{yu2024language} - combined with KnOTS or not – and TA. The hyperparameters (K, k) are set to be (7, 4).

As presented in Table~\ref{tab:lora}, \texttt{PACT-Iso-C} surpasses all methods in both experiments, achieving SOTA performance, which demonstrates the versatility of our merging approach. Meanwhile, \texttt{PACT-TA} also achieves performance close to KnOTS, a method specifically designed for LoRA fine-tuning, indicating that the LBW effect also exists in LoRA fine-tuning, even though the pre-trained parameters remain untouched during the fine-tuning process.

\begin{wraptable}[12]{r}{0.45\textwidth}
\centering
\scriptsize   
\setlength{\tabcolsep}{10pt}
\vspace{-0.6in}
\begin{tabular}{lcc}
\toprule
Method & ViT-B/32 & ViT-L/14 \\
\midrule
TIES          & 63.7 & 75.2 \\
DARE-TIES     & 63.7 & 74.7 \\
KnOTS-TIES    & 68.0 & 78.2 \\
KnOTS-DARE-TIES & 63.9 & 75.6 \\
Iso-CTS & \underline{73.7} & \underline{85.3} \\
\midrule
Task Arithmetic     & 63.7 & 74.4 \\
\rowcolor{pink!30}\texttt{PACT-TA} & 66.9 & 77.7 \\
Iso-C   & 73.6 & 83.7 \\
\rowcolor{pink!30}\texttt{PACT-Iso-C} & \textbf{75.1} & \textbf{86.3} \\
\bottomrule
\end{tabular}
\caption{Normalized per-task average accuracy. We merge 8 models fine-tuned with LoRA following~\citep{stoica2025model}. \texttt{PACT-Iso-C} achieves SOTA performance  on all evaluated scenarios.}
\label{tab:lora}
\vspace{-0.2in}
\end{wraptable}

\subsection{Efficient Substitute}

We evaluate the effect of replacing the full SVD in \texttt{PACT} with randomized SVD (RSVD)~\citep{halko2011finding}, as reported in Table~\ref{tab:rsvd}. Under exactly the same hyperparameters (K, k, $\alpha$), RSVD achieves performance on par with full SVD, and in one-third of the experiments it even surpasses the full SVD results, which is quite counterintuitive. We attribute this to the implicit denoising effect of random projection~\citep{halko2011finding}. Full SVD faithfully extracts all directions, including the potential overfitting noise hidden in $\Delta_t$; in contrast, the random Gaussian projection used in RSVD inherently acts as a low-pass filter. It robustly captures the main load-bearing wall components while naturally smoothing out high-frequency noise. Thus, RSVD provides a "free lunch" for \texttt{PACT}: as a structural regularizer, it reduces the complexity by several orders of magnitude while simultaneously enhancing the generalization ability of the merged model. We summarize and compare the computational complexity of \texttt{PACT} and its variant against these baselines in Appendix Sec.~\ref{sec:complexity}.

\begin{table}[htbp]
\centering
\scriptsize   
\setlength{\tabcolsep}{4pt}

\begin{tabular}{lccccccccc}
\toprule
\multirow{2}{*}{Method} & \multicolumn{3}{c}{ViT-B/32} & \multicolumn{3}{c}{ViT-B/16} & \multicolumn{3}{c}{ViT-L/14} \\
\cmidrule(lr){2-4} \cmidrule(lr){5-7} \cmidrule(lr){8-10}
 & 8 tasks & 14 tasks & 20 tasks & 8 tasks & 14 tasks & 20 tasks & 8 tasks & 14 tasks & 20 tasks \\
\midrule
TSV-M   & 85.8(92.3) & 78.6(85.9) & 76.3(83.4) & 88.8(93.8) & 83.1(89.1) & 80.0(85.6) & 93.0(97.1)
 & 87.7(92.6) & 87.0(91.5) \\
\rowcolor{lightgreen} RSVD & 85.5(92.0)& \textbf{78.8(86.2)} & 76.1(83.4) & 88.7(93.6) & 82.9(88.9) & 79.8(85.4) & 93.0(97.0)
 & 87.6(92.5) & 86.9(91.5) \\
\midrule
\texttt{PACT-TA}   & 75.2(81.0) & 66.4(73.3) & 59.9(66.4) & 85.5(90.3) & 76.8(82.7) & 68.7(74.0) & 92.8(96.9) & 87.0(92.0) & 83.0(87.7) \\
\rowcolor{lightgreen}RSVD & \textbf{75.6(81.5)} & 66.4(73.3) & 59.9(66.4) & \textbf{85.6(90.5)} & 76.8(82.7) & 68.7(74.0) & \textbf{92.9(97.0)} & \textbf{87.3(92.2)} & 83.0(87.6) \\
\midrule
\texttt{PACT-IsoC}   & 87.3(94.1) & 81.3(89.0) & 78.7(86.0) & 92.0(97.2) & 86.1(92.3) & 83.4(89.3) & 95.1(98.3) & 90.5(95.5) & 90.3(95.0) \\
\rowcolor{lightgreen}RSVD & 87.3(94.1) & \textbf{81.9(89.7)} & 78.7(86.0) & 92.0(97.3) & \textbf{86.3(92.7)} & 83.2(89.1) & \textbf{95.2(99.3)} & 90.5(95.5) & 90.3(95.0) \\
\midrule
\texttt{PACT-TSVM}   & 88.1(94.8) & 82.0(89.6) & 80.3(87.6) & 91.2(96.3) & 85.4(91.5) & 83.8(89.6) & 94.4(98.5) & 89.5(94.4) & 89.7(94.3) \\
\rowcolor{lightgreen}RSVD   & 88.1(94.8) & 81.8(89.5) & 80.0(87.2) & 91.1(96.2) & 85.3(91.4) & \textbf{83.9(89.7)} & \textbf{94.5(98.6)} & 89.5(94.3) & 89.7(94.3) \\
\bottomrule
\end{tabular}
\caption{Results of replacing SVD with RSVD. We present average absolute accuracy and average normalized accuracy (in bracket) in $\%$. Performance improvements are shown in \textbf{bold}.}
\label{tab:rsvd}
\end{table}

\section{Conclusion and Future Work}
\label{sec:conclusion}

In this work, we have addressed a fundamentally overlooked assumption in existing task-vector-based model merging paradigms---namely, that pre-trained base parameters remain inert and task performance depends solely on task-vector modifications. Through microscopic Fisher information analysis and macroscopic SVD subspace projections, we have identified and causally validated the existence of LBW dimensions. These low-dimensional, high-energy pre-trained structures remain untouched during downstream fine-tuning, yet are critically essential for preserving task-specific inference capabilities. Our extensive causal ablation experiments demonstrate that a mere $1.17\%$ parametric disruption to these localized pre-trained core structures leads to a catastrophic collapse of downstream performance, establishing their indispensable structural role.

To mitigate the severe subspace collisions on these LBW during unconstrained merging, we proposed \texttt{PACT}. By elegantly constructing a global orthogonal shield, \texttt{PACT} surgically filters out interfering update components of other tasks from each expert's load-bearing subspace. This data-free geometric approach is highly modular and compatible with all existing task-vector-based merging methods. Empirical evaluations across 20 diverse vision tasks, spanning different backbones (ViT-B and ViT-L) and adaptation paradigms (full fine-tuning and LoRA), demonstrate that PACT-enhanced merging algorithms consistently outperform baseline methods. Furthermore, our efficient variant based on randomized SVD successfully decouples the cubic complexity barrier from the number of tasks, ensuring strong scalability for large-scale merging scenarios.

Our findings offer a new, geometer-centric perspective on the mechanics of deep neural networks, suggesting that model merging is not merely about resolving conflicts among task updates, but actively respecting and shielding the critical pre-trained foundations. Moving forward, several promising research directions emerge. First, extending the \texttt{PACT} framework to large language models (LLMs) and investigating the capacity limits of protected subspaces under hundreds of tasks remains a critical next step. Second, exploring non-linear projection operators and adaptive, weight-importance-aware dimensionalities for $V_{t,LBW}$ could further refine the surgical precision of the orthogonal shield. We hope this work inspires deeper exploration into the geometric structures of pre-trained models and paves the way for more robust multi-task model synthesis. We also summarize our limitations in Appendix Sec.~\ref{sec:limit}.
\newpage

\bibliography{iclr2026_conference}
\bibliographystyle{iclr2026_conference}

\clearpage
\appendix

\setcounter{tocdepth}{2}
\startcontents[appendix]
\printcontents[appendix]{}{1}{}

\newpage

\section{Proofs of Theoretical Properties for \texttt{PACT}}
\label{app:theoretical_proofs}

\subsection{Proof of Theorem \ref{thm:projected_performance_preservation}} 
\begin{proof}
    Let 
    \begin{equation} 
        D_j = \Delta_j-\tilde{\Delta}_j = \Delta_j P_jP_j^\top 
    \end{equation} be the component discarded by \texttt{PACT}. Then \begin{equation} 
        \theta_0+\tilde{\Delta}_j = \theta_j-D_j . 
    \end{equation}
    Since $\mathcal{L}_j$ is locally $\beta_j$-smooth around $\theta_j$, we have 
    \begin{equation} 
        \mathcal{L}_j(\theta_j-D_j) \leq \mathcal{L}_j(\theta_j) - \left\langle \nabla \mathcal{L}_j(\theta_j), D_j \right\rangle + \frac{\beta_j}{2} \|D_j\|_F^2 . 
    \end{equation} 
    Substituting $D_j=\Delta_jP_jP_j^\top$ gives 
    \begin{equation} 
    \begin{aligned} 
        - \left\langle \nabla \mathcal{L}_j(\theta_j), D_j \right\rangle &\leq \left| \left\langle \nabla \mathcal{L}_j(\theta_j), \Delta_jP_jP_j^\top \right\rangle \right| \\ &= \left| \left\langle \nabla \mathcal{L}_j(\theta_j)P_j, \Delta_jP_j \right\rangle \right| \\ &\leq \left\| \nabla \mathcal{L}_j(\theta_j)P_j \right\|_F \left\| \Delta_jP_j \right\|_F , 
    \end{aligned} 
    \end{equation} where the last step follows from Cauchy--Schwarz. Moreover, since $P_j$ has orthonormal columns, 
    \begin{equation} 
        \|D_j\|_F = \|\Delta_jP_jP_j^\top\|_F = \|\Delta_jP_j\|_F . \end{equation} 
        Combining the above inequalities yields 
    \begin{equation} 
    \begin{aligned} 
    \mathcal{L}_j(\theta_0+\tilde{\Delta}_j) - \mathcal{L}_j(\theta_0+\Delta_j) &= \mathcal{L}_j(\theta_j-D_j) - \mathcal{L}_j(\theta_j) \\ &\leq \left\| \nabla \mathcal{L}_j(\theta_j)P_j \right\|_F \left\| \Delta_jP_j \right\|_F + \frac{\beta_j}{2} \left\| \Delta_jP_j \right\|_F^2 . 
    \end{aligned} 
    \end{equation}
\end{proof}

\subsection{Proof of Corollary \ref{cor:active_subspace_relaxation}}
\begin{proof}
Let $P_j=V_{j,k}V_{j,k}^{\top}$. By the triangle inequality, \begin{equation} 
\begin{aligned} 
    \left\| \nabla \mathcal{L}_j(\theta_j)P_j \right\|_F &\leq \left\| \nabla \mathcal{L}_j(\theta_j)P_jP_j \right\|_F + \left\| \nabla \mathcal{L}_j(\theta_j)(I-P_j)P_j \right\|_F \\ &\leq \left\| \nabla \mathcal{L}_j(\theta_j) \right\|_F \left\| V_{j,k}^{\top}P_j \right\|_2 + \epsilon_g \left\| \nabla \mathcal{L}_j(\theta_j) \right\|_F \\ &= \left\| \nabla \mathcal{L}_j(\theta_j) \right\|_F \left( \cos(\phi_j)+\epsilon_g \right). 
\end{aligned} 
\end{equation} Similarly, 
\begin{equation} 
    \left\| \Delta_jP_j \right\|_F \leq \left\| \Delta_j \right\|_F \left( \cos(\phi_j)+\epsilon_\Delta \right). 
\end{equation} 
Substituting these two inequalities into Theorem~\ref{thm:projected_performance_preservation} gives Eqn.~\eqref{eq:angle_relaxed_bound}. When $\epsilon_g=\epsilon_\Delta=0$, the first-order term becomes 
\begin{equation} 
  \Delta \mathcal{L}_j^{(1)} \leq \left\| \nabla \mathcal{L}_j(\theta_j) \right\|_F \left\| \Delta_j \right\|_F \cos^2(\phi_j), 
\end{equation}    
\end{proof}

\subsection{Derivation and Proof of the Ablation Approximation}
\label{appendix:proof}

To justify the ablation approximation presented in Eqn.~\ref{ablation_approx}, we present a unified mathematical derivation based on the construction of the LBW subspace. Recall that the task-specific model parameter is decomposed as $\theta_t = \theta_0 + \Delta_t$, and the ablation operation is defined as:
\begin{equation}\label{eq:abla_def}
\theta_{abla} = \theta_t (I - V_{t,LBW}V_{t,LBW}^{\top})
\end{equation}
where the LBW subspace basis $V_{t,LBW} \in \mathbb{R}^{n \times d_{LBW}}$ is constructed via Eqn.\ref{pact1}.

By this construction, $V_{t,LBW}$ lies within the column space of $(I - V_{t,k}V_{t,k}^{\top})V_{0,K}$. Consequently, we can express $V_{t,LBW} = (I - V_{t,k}V_{t,k}^{\top})V_{0,K} R^{-1}$ for some invertible matrix $R$. Left-multiplying this relation by $V_{t,k}^{\top}$ yields:
\begin{align}
V_{t,k}^{\top} V_{t,LBW} &= V_{t,k}^{\top} (I - V_{t,k}V_{t,k}^{\top}) V_{0,K} R^{-1} \nonumber \\
&= (V_{t,k}^{\top} - V_{t,k}^{\top} V_{t,k} V_{t,k}^{\top}) V_{0,K} R^{-1}
\end{align}
Substituting the orthonormality condition $V_{t,k}^{\top} V_{t,k} = I$, the expression simplifies to:
\begin{align}
V_{t,k}^{\top} V_{t,LBW} &= (V_{t,k}^{\top} - I \cdot V_{t,k}^{\top}) V_{0,K} R^{-1} \nonumber \\
&= 0 \cdot V_{0,K} R^{-1} = 0
\end{align}
This establishes that the constructed LBW subspace $V_{t,LBW}$ is orthogonal to the principal task-specific subspace $V_{t,k}$.

Next, we assume that the task-specific parameter update $\Delta_t$ primarily aligns with the subspace spanned by its principal components $V_{t,k}$. This relationship can be formally expressed as the approximation $\Delta_t \approx \Delta_t V_{t,k} V_{t,k}^{\top}$. Projecting the update $\Delta_t$ onto the LBW subspace then yields:
\begin{align}
\Delta_t V_{t,LBW} &\approx \left( \Delta_t V_{t,k} V_{t,k}^{\top} \right) V_{t,LBW} \nonumber \\
&= \Delta_t V_{t,k} \left( V_{t,k}^{\top} V_{t,LBW} \right)
\end{align}
Applying the strict orthogonality $V_{t,k}^{\top} V_{t,LBW} = 0$ derived above, we immediately obtain:
\begin{equation}\label{eq:approx_orth}
\Delta_t V_{t,LBW} \approx 0 \quad \text{and thus} \quad \Delta_t V_{t,LBW} V_{t,LBW}^{\top} \approx 0
\end{equation}

Finally, by substituting the decomposition $\theta_t = \theta_0 + \Delta_t$ into the ablation definition in Eqn.~\ref{eq:abla_def} and expanding the terms, we have:
\begin{align}
\theta_{abla} &= (\theta_0 + \Delta_t) (I - V_{t,LBW}V_{t,LBW}^{\top}) \nonumber \\
&= \theta_0 (I - V_{t,LBW}V_{t,LBW}^{\top}) + \Delta_t - \Delta_t V_{t,LBW} V_{t,LBW}^{\top}
\end{align}
Using the approximation $\Delta_t V_{t,LBW} V_{t,LBW}^{\top} \approx 0$ from Eqn.~\ref{eq:approx_orth}, the expression simplifies to:
\begin{equation}
\theta_{abla} \approx \theta_0 (I - V_{t,LBW} V_{t,LBW}^\top) + \Delta_t
\end{equation}
This completes the derivation, demonstrating that the ablation operation selectively filters out pre-trained components within the LBW subspace while leaving the task-specific updates largely intact.
\section{Computational complexity analysis}
\label{sec:complexity}

In this Section, we analyze the computational complexity of our proposed \texttt{PACT-Iso-C} and \texttt{PACT-TSV-M} merging algorithms, alongside their highly efficient RSVD variants~\citep{halko2011finding}. We compare them with the baseline Iso-C and the SOTA subspace merging methods, TSV-M \citep{gargiulo2025task} and Iso-CTS \citep{marczak2025no}.

Following the conventions established in previous works, let $\Delta_t \in \mathbb{R}^{n \times n}$, and let $T$ and $L$ be the number of tasks and network layers, respectively. For simplicity, assume that each layer consists of a single squared $n \times n$ matrix. 

In our analysis, we focus on the number of Singular Value Decompositions (SVDs) performed by each algorithm, as this is by far the most costly component of each pipeline. The complexity of a single full SVD on $\Delta_t \in \mathbb{R}^{n \times n}$ is $\mathcal{O}(n^3)$ \citep{vasudevan2017hierarchical}. Below, we detail the total computational complexity for each merging method and summarize them in Table~\ref{tab:complexity}:

\begin{itemize}
    \item \textbf{Iso-C} performs a single SVD on the aggregated $\Delta_{TA}$ per layer, with total complexity:
    \begin{equation}
        \mathcal{O}(\text{Iso-C}) = \mathcal{O}(L n^3)
    \end{equation}
    
    \item \textbf{TSV-M} performs:
    \begin{itemize}
        \item $T$ SVDs per layer on each task matrix, yielding $\mathcal{O}(T L n^3)$.
        \item Two additional SVDs per layer for subspace alignment, yielding $\mathcal{O}(2 L n^3)$.
    \end{itemize}
    Yielding the total complexity:
    \begin{equation}
        \mathcal{O}(\text{TSV-M}) = \mathcal{O}(T L n^3 + 2 L n^3) = \mathcal{O}\big((T + 2) L n^3\big) = \mathcal{O}(T L n^3)
    \end{equation}

    \item \textbf{Iso-CTS} performs:
    \begin{itemize}
        \item One SVD on $\Delta_{TA}$ per layer, with complexity $\mathcal{O}(L n^3)$.
        \item One SVD on each $\Delta_t$, for all $T$ tasks per layer, with complexity $\mathcal{O}(T L n^3)$.
        \item Two SVDs on two constructed matrices $U^*, V^* \in \mathbb{R}^{n \times n}$ per layer, yielding $\mathcal{O}(2 L n^3)$.
    \end{itemize}
    Therefore, the total complexity equals:
    \begin{equation}
        \mathcal{O}(\text{Iso-CTS}) = \mathcal{O}(L n^3 + T L n^3 + 2 L n^3) = \mathcal{O}\big((T + 3) L n^3\big) = \mathcal{O}(T L n^3)
    \end{equation}

    \item \texttt{PACT-Iso-C} (Ours) integrates our orthogonal projection filtering with the Iso-C framework. It performs:
    \begin{itemize}
        \item $(T+1)$ SVDs for the \texttt{PACT} filtering phase (one on $\theta_0$ and $T$ on task matrices $\Delta_t$): $\mathcal{O}\big((T+1) L n^3\big)$.
        \item One SVD on the aggregated filtered matrix $\sum \tilde{\Delta}_j$ for the final Iso-C merging: $\mathcal{O}(L n^3)$.
    \end{itemize}
    Yielding the total complexity:
    \begin{equation}
        \mathcal{O}(\text{\texttt{PACT-Iso-C}}) = \mathcal{O}\big((T + 1) L n^3 + L n^3\big) = \mathcal{O}\big((T + 2) L n^3\big) = \mathcal{O}(T L n^3)
    \end{equation}
    *(Note: The QR decompositions required to build the protection spaces have a complexity of $\mathcal{O}(n (TK)^2)$, which is strictly dominated by the $\mathcal{O}(n^3)$ SVD terms since $K \ll n$.)*

    \item \texttt{PACT-TSV-M} (Ours) integrates \texttt{PACT} filtering with TSV-M merging. It performs:
    \begin{itemize}
        \item $(T+1)$ SVDs for the \texttt{PACT} filtering phase: $\mathcal{O}\big((T+1) L n^3\big)$.
        \item $(T+2)$ SVDs applied to the filtered matrices $\tilde{\Delta}_t$ for the TSV-M merging phase: $\mathcal{O}\big((T+2) L n^3\big)$.
    \end{itemize}
    Yielding the total complexity:
    \begin{equation}
        \mathcal{O}(\text{\texttt{PACT-TSV-M}}) = \mathcal{O}\big((T+1) L n^3 + (T+2) L n^3\big) = \mathcal{O}\big((2T + 3) L n^3\big) = \mathcal{O}(T L n^3)
    \end{equation}

    \vspace{0.2cm}
    \noindent \textbf{Efficient Variants with Randomized SVD (RSVD):}
    \\
    Since extracting task-specific features only requires the top-$k$ singular vectors, RSVD can approximate these components using random projections, strictly dropping the decomposition complexity from $\mathcal{O}(n^3)$ to $\mathcal{O}(n^2 k_{max})$, where $k_{max}$ is the maximum rank retained. However, SVDs utilized for final subspace alignment or full aggregation must remain exact.
    \begin{itemize}
        \item \textbf{Efficient TSV-M (RSVD)} replaces the first $T$ task-specific SVDs, but the 2 final alignment SVDs cannot be substituted:
        \begin{equation}
            \mathcal{O}(\text{Efficient TSV-M}) = \mathcal{O}\big(T L n^2 k_{max} + 2 L n^3\big)
        \end{equation}
        \item \textbf{Efficient \texttt{PACT-Iso-C} (RSVD)} replaces all $(T+1)$ SVDs in the \texttt{PACT} phase, leaving only the final Iso-C aggregation SVD exact:
        \begin{equation}
            \mathcal{O}(\text{Efficient \texttt{PACT-Iso-C}}) = \mathcal{O}\big((T + 1) L n^2 k_{max} + L n^3\big)
        \end{equation}
        \item \textbf{Efficient \texttt{PACT-TSV-M} (RSVD)} replaces the $(T+1)$ SVDs in the \texttt{PACT} phase and the $T$ task SVDs in the TSV-M phase, while the 2 final TSV-M alignment SVDs remain exact:
        \begin{equation}
            \mathcal{O}(\text{Efficient \texttt{PACT-TSV-M}}) = \mathcal{O}\big((2T + 1) L n^2 k_{max} + 2 L n^3\big)
        \end{equation}
    \end{itemize}
\end{itemize}

\textbf{Comparison Summary.} As highlighted in the analysis, all exact subspace methods (TSV-M, Iso-CTS, \texttt{PACT-Iso-C}, and \texttt{PACT-TSV-M}) share the same asymptotic complexity scaling of $\mathcal{O}(T L n^3)$. Notably, our exact \texttt{PACT-Iso-C} strictly requires $(T+2)$ full SVDs, posing slightly less overhead than the $(T+3)$ required by Iso-CTS. Exact \texttt{PACT-TSV-M} unsurprisingly carries the highest constant factor with $(2T+3)$ SVDs due to the sequential combination of two distinct algorithms.

More importantly, our \textbf{Efficient RSVD Variants} fundamentally decouple the cubic complexity barrier from the number of tasks $T$. While a small, constant number of full SVDs (e.g., 1 for Iso-C, 2 for TSV-M) remains inescapable for the final aggregation steps, substituting the $T$-dependent operations with $\mathcal{O}(n^2 k_{max})$ practically eliminates the scaling bottleneck. Because the number of active vectors $k_{max}$ is a tiny fraction of the matrix dimension ($k_{max} \ll n$), the computational cost scaling with $T$ is reduced by orders of magnitude. This makes the heavily piped Efficient \texttt{PACT-TSV-M} highly scalable and drastically cheaper to compute than traditional exact merging pipelines on large-scale models.

\begin{table}[htbp]
\centering
\resizebox{\textwidth}{!}{%
\begin{tabular}{lccc}
\toprule
\textbf{Method} & \textbf{\# Full SVDs} ($\mathcal{O}(n^3)$) & \textbf{\# RSVDs} ($\mathcal{O}(n^2 k_{max})$) & \textbf{Total Complexity} \\ 
\midrule
\multicolumn{4}{c}{\textit{Exact Methods (Baselines \& Ours)}} \\
\midrule
Iso-C \citep{marczak2025no} & $1$ & $0$ & $\mathcal{O}(L n^3)$ \\
TSV-M \citep{gargiulo2025task} & $T + 2$ & $0$ & $\mathcal{O}\big((T + 2) L n^3\big)$ \\
Iso-CTS \citep{marczak2025no} & $T + 3$ & $0$ & $\mathcal{O}\big((T + 3) L n^3\big)$ \\
\textbf{PACT-IsoC (Ours)} & $T + 2$ & $0$ & $\mathcal{O}\big((T + 2) L n^3\big)$ \\
\textbf{PACT-TSVM (Ours)} & $2T + 3$ & $0$ & $\mathcal{O}\big((2T + 3) L n^3\big)$ \\
\midrule
\multicolumn{4}{c}{\textit{Efficient Variants with Randomized SVD (RSVD)}} \\
\midrule
Efficient TSV-M & $2$ & $T$ & $\mathcal{O}\big(T L n^2 k_{max} + 2 L n^3\big)$ \\
\textbf{Efficient PACT-IsoC (Ours)} & $1$ & $T + 1$ & $\mathcal{O}\big((T + 1) L n^2 k_{max} + L n^3\big)$ \\
\textbf{Efficient PACT-TSVM (Ours)} & $2$ & $2T + 1$ & $\mathcal{O}\big((2T + 1) L n^2 k_{max} + 2 L n^3\big)$ \\
\bottomrule
\end{tabular}%
}
\caption{Comparison of computational complexity among different merging methods. $T$ and $L$ denote the number of tasks and layers, respectively. $n$ is the dimension of the squared layer matrices, and $k_{max}$ ($k_{max} \ll n$) is the maximum number of vectors retained in RSVD approximations.}
\label{tab:complexity}
\end{table}
\section{Additional Experiments}

\subsection{Extended Analysis on Fisher-level Parameter Collisions}
\label{app:fisher_analysis}

\begin{table}[htbp]
\centering
\begin{tabular}{c|cc>{\columncolor{pink!30}}ccc}
\toprule
$$(A,B)$$ & $M0$ & $M1$ & $M2$ & $M3$ & $M4$ \\
\midrule

(SUN397, Cars) & 81.32 & 81.22 & 58.40 & 77.03 & 73.58\\
(Cars, SUN397) & 86.49 & 86.61 & 5.04 & 77.93 & 70.64 \\

(GTSRB, MNIST) & 99.92 & 99.92 & 28.04 & 93.39 & 90.47 \\
(MNIST, GTSRB) & 99.50 & 99.50 & 39.52 & 99.20 & 96.69 \\

(DTD, RESISC45) & 98.40 & 98.40 & 92.02 & 97.69 & 97.34 \\
(RESISC45, DTD) & 97.04 & 97.09 & 74.94 & 96.38 & 94.94 \\

(Cars, PACM) & 86.49 & 86.36 & 2.70 & 33.37 & 51.76 \\
(PACM, Cars) & 97.86 & 97.84 & 94.05 & 97.76 & 97.80 \\

(Food101, Flower102) & 89.56 & 89.60 & 55.29 & 84.01 & 83.16 \\
(Flower102, Food101) & 97.06 & 97.06 & 62.75 & 93.14 & 90.20 \\

(GTSRB, DTD) & 99.92 & 99.92 & 77.03 & 99.49 & 99.77 \\
(DTD, GTSRB) & 98.40 & 98.40 & 50.71 & 94.50 & 90.60 \\

(SUN397, OxfordIIITPet) & 81.32 & 81.17 & 55.83 & 78.32 & 77.45 \\
(OxfordIIITPet, SUN397) & 94.84 & 94.84 & 51.54 & 93.03 & 91.94 \\

(CIFAR100, RenderedSST2) & 91.08 & 91.12 & 87.40 & 90.25 & 90.06 \\
(RenderedSST2, CIFAR100) & 77.31 & 77.46 & 72.64 & 77.02 & 76.88 \\

\bottomrule
\end{tabular}
\caption{Accuracy (\%) under different parameters modifications. The value of $\alpha$ is set to 30.0.}
\label{tab:attk1_detail}
\end{table}

Our parameter-level perturbation experiments, as detailed in Section~\ref{sec:weight_attc}, are extended here across 16 diverse task pairs $(A,B)$ to substantiate the universality of the LBW hypothesis. Table~\ref{tab:attk1_detail} presents the downstream performance of task $A$ under different modification schemes with an attack strength of $\alpha=30$. We observe a consistent trend across all evaluated pairs:
\paragraph{Severe Degradation under Crucial Attack ($M2$)} Modifying the parameters within the crucial mask ($M2$) leads to a drastic performance drop. For instance, in the (Cars, SUN397) pair, the accuracy of task $A$ (Cars) collapses from a baseline of 86.49\% to 5.04\%. Similarly, for (Cars, PACM), the accuracy drops to 2.70\%. This massive collapse indicates that these small-magnitude, high-Fisher parameters are highly sensitive and hold critical task-specific information.
\paragraph{Robustness under Safe and Random Attacks ($M3$, $M4$)} Conversely, modifying the parameters in the safe mask ($M3$) or the random mask ($M4$) with the same intensity yields considerably milder degradation. For example, under $M3$ and $M4$, the performance on Cars in the (Cars, SUN397) pair remains at 77.93\% and 70.64\%, respectively. These results confirm that the "load-bearing" property is highly localized and successfully captured by our Fisher-based identification method.

To confirm whether model merging interferes with these LBW dimensions, we evaluate the effect of restoring the crucial mask parameters to their pre-trained values $\theta_0$ after a naive merging of task vectors $\theta_0+\Delta_A+\Delta_B$. As shown in Table~\ref{tab:attk2_detail}, across almost all task pairs, simply resetting the crucial mask parameters (which constitute only $0.80\%-2.97\%$ of the total parameters) back to $\theta_0$ yields a measurable performance recovery $\Delta_{acc}$. Specifically, we observe an improvement of $+2.33\%$ for (Cars, PACM) and $+2.13\%$ for (DTD, GTSRB). This empirical recovery demonstrates that the task vector of the counterpart task ($\Delta_B$) indeed introduces destructive interference to the crucial pre-trained weights of task A during naive addition.

\begin{table}[htbp]
\centering
\begin{tabular}{c|ccc>{\columncolor{pink!30}}c}
\toprule
$$(A,B)$$ & crucial mask ratio(\%) & $\theta_0+\Delta_A+\Delta_B$ & $\theta_0$ on crucial  & $\Delta_{acc}$ \\
\midrule
(SUN397, Cars) & 1.28 & 70.97 & 71.78  & \color{red}{0.81} \\
(Cars, SUN397) & 1.52 & 78.62 & 79.36  & \color{red}{0.74} \\

(GTSRB, MNIST) & 2.00 & 91.55 & 93.47 & \color{red}{1.91} \\
(MNIST, GTSRB) & 2.97 & 99.08 & 99.14 & \color{red}{0.06} \\

(DTD, RESISC45) & 1.19 & 94.15 & 94.68 & \color{red}{0.53} \\
(RESISC45, DTD) & 1.48 & 94.29 & 94.60 & \color{red}{0.32} \\

(Cars, PACM) &1.89& 74.69 & 77.03 & \color{red}{2.33} \\
(PACM, Cars) &0.80& 97.50 & 97.46 & \color{blue}{0.04} \\

(Food101, Flower102) &2.20& 84.92 & 85.46 & \color{red}{0.54} \\
(Flower102, Food101) &1.12& 93.14 & 93.14 & 0.00 \\

(GTSRB, DTD) & 1.91 & 99.89 & 99.89 & 0.00 \\
(DTD, GTSRB) & 2.17 & 92.55 & 94.68 & \color{red}{2.13} \\

(SUN397, OxfordIIITPet) & 1.28 & 76.38 & 76.88 & \color{red}{0.50} \\
(OxfordIIITPet, SUN397) & 1.87 & 92.66 & 92.66 & 0.00 \\

(CIFAR100, RenderedSST2) &0.92& 89.02 & 88.98 & \color{blue}{0.04} \\
(RenderedSST2, CIFAR100) &1.26& 76.16 & 76.01 & \color{blue}{0.14} \\

\bottomrule
\end{tabular}
\caption{$\Delta_{acc}$ represents the accuracy (\%) change due to the crucial mask replacement, with improvement in red and decrease in blue.}
\label{tab:attk2_detail}
\end{table}

To thoroughly evaluate the sensitivity of the identified load-bearing parameters, we visualize the downstream accuracy of task A under varying attack strengths $\alpha\in[1.0,5.0,10.0,20.0,30.0]$ across all 15 evaluated task pairs in Figure~\ref{fig:15images}.
The curves show a systematic and consistent divergence. Across all 16 pairs, the $M2$ curve (representing the targeted attack on the crucial mask) drops precipitously as the attack strength $\alpha$ increases. In contrast, the $M3$ (safe mask) and $M4$ (random mask) curves exhibit a much gentler decline, remaining relatively stable even under high perturbation levels. This uniform behavior across diverse task combinations suggests that the load-bearing wall phenomenon is not an artifact of a specific dataset pair, but rather a structural characteristic inherent to pre-trained vision models.

\subsection{Extended Geometric Subspace Analysis and Layer-wise Heterogeneity}
\label{app:geometric_subspace}
This subsection provides the complete, expanded experimental results for the global subspace ablation analysis introduced in Section~\ref{sec:subspace_analysis} of the main text. While Table~\ref{tab:attk3} in Section~\ref{sec:subspace_analysis} presents a concise subset of six representative tasks due to space constraints, we extend this evaluation here to the full set of 19 diverse downstream datasets on the ViT-B-16 model, as summarized in Table~\ref{tab:attk3_detail}. The results consistently indicate that surgically removing the low-dimensional $V_{t,LBW}$ subspace (which represents on average only $1.17\%$ of the total parameter dimensions) leads to a severe degradation of downstream performance across all 18 tasks. For instance, CIFAR100 accuracy falls from $91.08\%$ to $1.88\%$, and Flowers102 drops from $97.06\%$ to $0.98\%$. In contrast, ablating a random orthogonal control subspace of the identical dimensionality ($V_{rand}$) causes negligible performance changes, retaining $88.15\%$ and $91.18\%$ accuracy respectively. The ratio of the resulting performance drops ($\Delta_{acc}^{LBW} / \Delta_{acc}^{rand}$) ranges from $13.75\times$ to over $3104\times$. This notable asymmetry provides supportive empirical evidence indicating that these low-rank pre-trained subspaces are highly critical for sustaining downstream capabilities.
\begin{figure*}
\centering

\subfloat{\includegraphics[width=0.245\textwidth]{fig/fisher/_Cars_SUN397_.pdf}}%
\subfloat{\includegraphics[width=0.245\textwidth]{fig/fisher/_SUN397_Cars_.pdf}}%
\hfill
\subfloat{\includegraphics[width=0.245\textwidth]{fig/fisher/_GTSRB_MNIST_.pdf}}%
\hfill
\subfloat{\includegraphics[width=0.245\textwidth]{fig/fisher/_MNIST_GTSRB_.pdf}} \\

\subfloat{\includegraphics[width=0.245\textwidth]{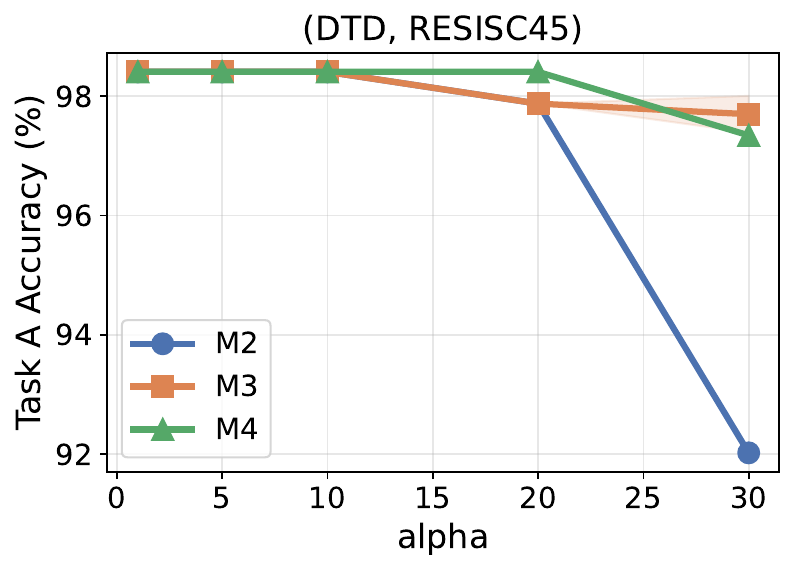}}%
\hfill
\subfloat{\includegraphics[width=0.245\textwidth]{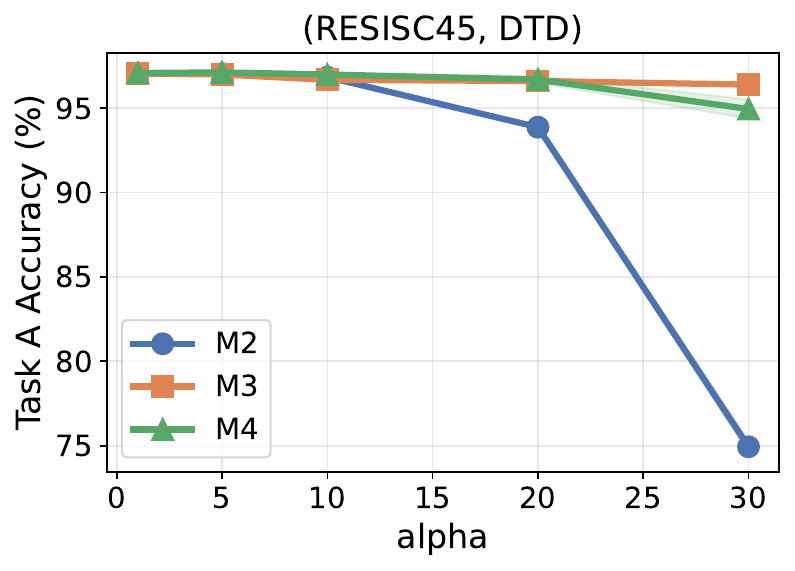}} %
\hfill
\subfloat{\includegraphics[width=0.245\textwidth]{fig/fisher/_Cars_PACM_.pdf}} 
\hfill
\subfloat{\includegraphics[width=0.245\textwidth]{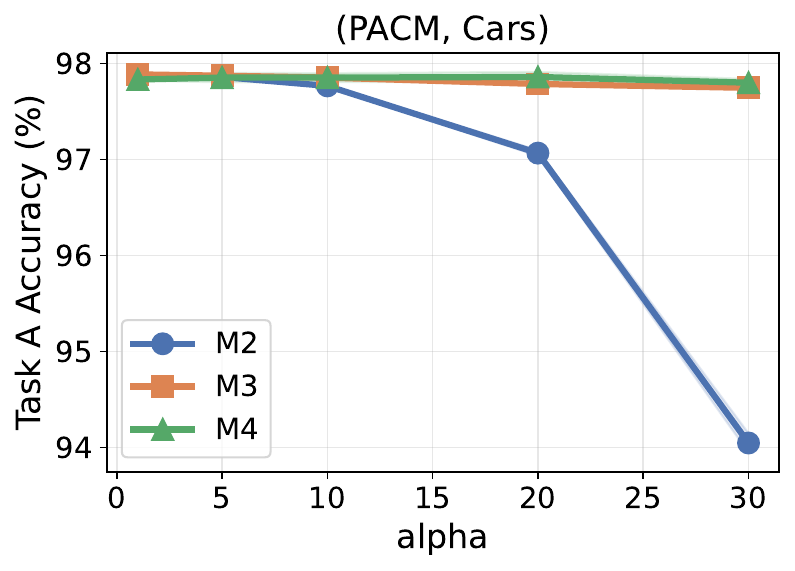}}\\

\subfloat{\includegraphics[width=0.245\textwidth]{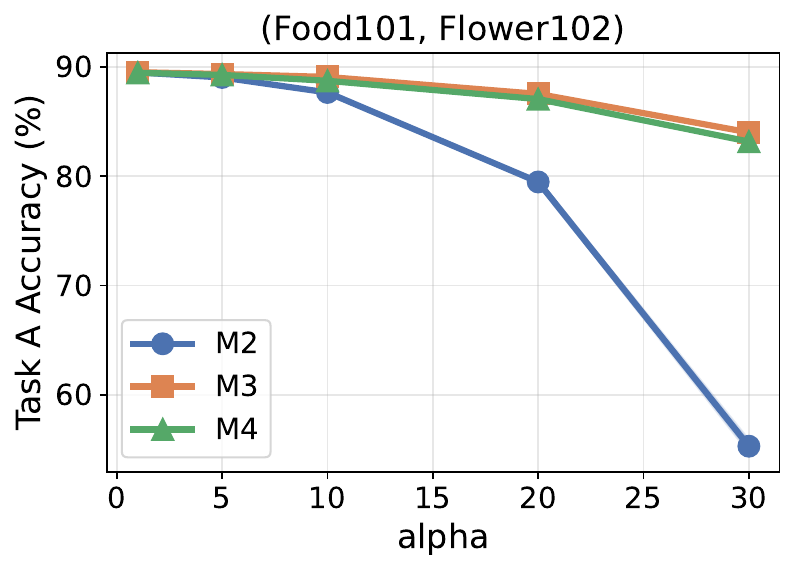}}%
\hfill
\subfloat{\includegraphics[width=0.245\textwidth]{fig/fisher/_Flower102_Food101_.pdf}} 
\hfill
\subfloat{\includegraphics[width=0.245\textwidth]{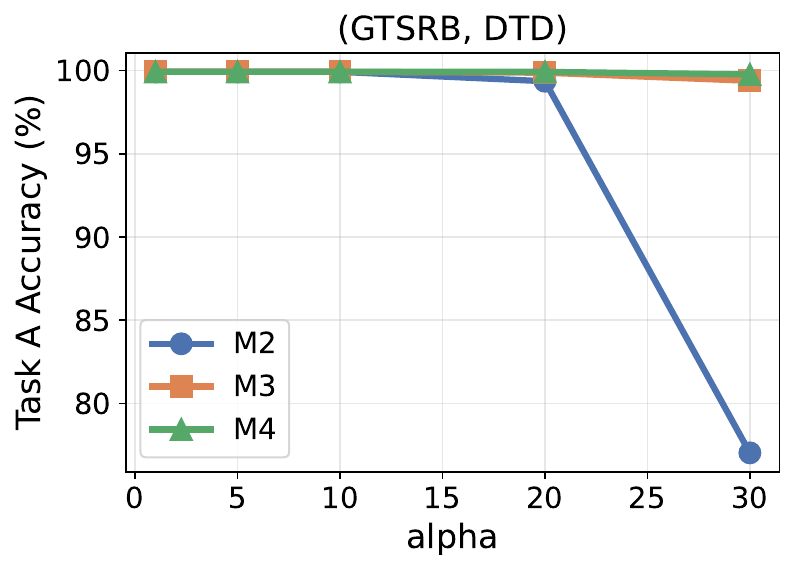}} 
\hfill
\subfloat{\includegraphics[width=0.245\textwidth]{fig/fisher/_DTD_GTSRB_.pdf}}  \\

\subfloat{\includegraphics[width=0.245\textwidth]{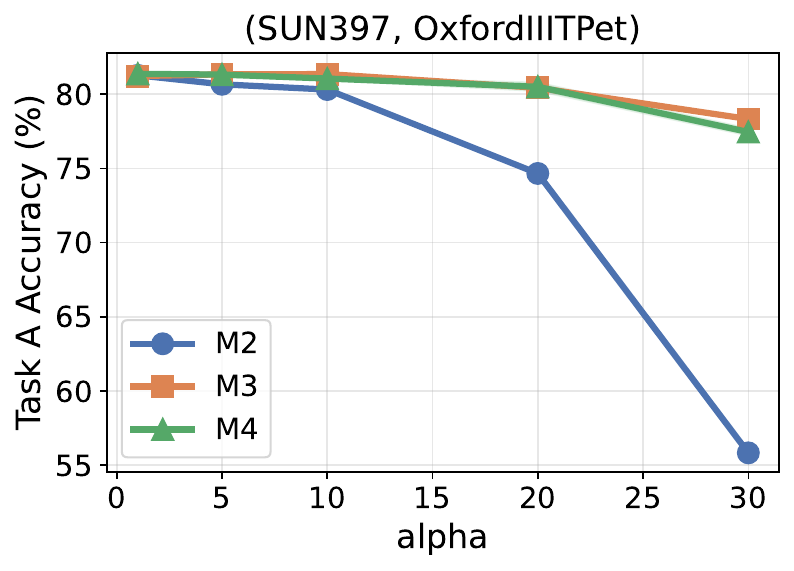}} 
\hfill
\subfloat{\includegraphics[width=0.245\textwidth]{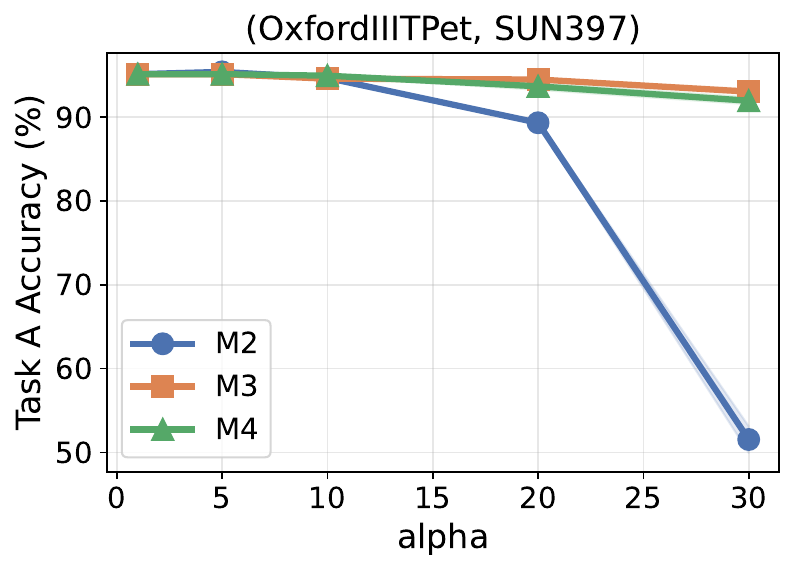}}%
\hfill
\subfloat{\includegraphics[width=0.245\textwidth]{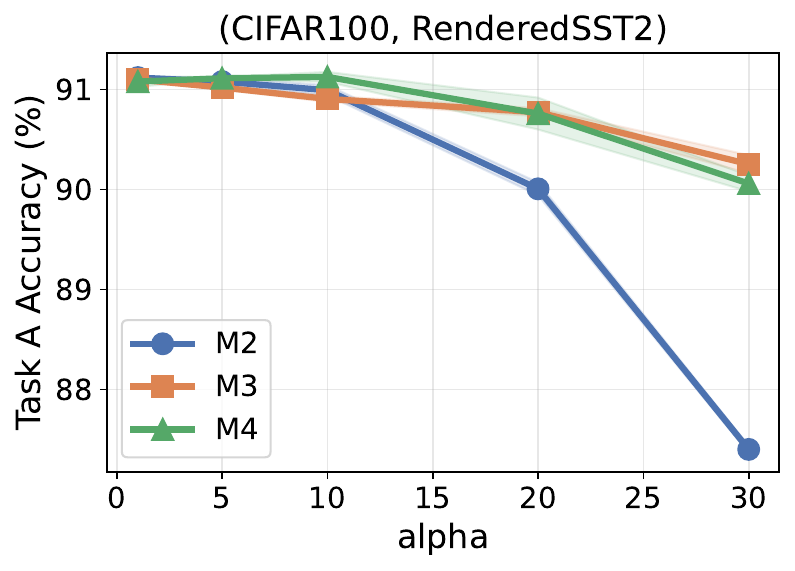}}%
\hfill
\subfloat{\includegraphics[width=0.245\textwidth]{fig/fisher/_RenderedSST2_CIFAR100_.pdf}} \\

\caption{Accuracy (\%) under varying attack strengths with respective to different masks across 16 task pairs. Generally, $M2$ is could be easily attacked while $M3$ and $M4$ are rather stable, indicating the importance of parameters within crucial mask.} 
\label{fig:15images}
\end{figure*}

To analyze the underlying geometric conflicts causing these performance drops, we supplement the main text experiments by evaluating the pairwise Sacred Space Similarity ($Sim$, Eqn.~\ref{sim_metric}) and Hidden Interference Ratio ($Inf$, Eqn.~\ref{sim_metric}). When these metrics are averaged globally across all network layers (Figures \ref{fig:global_avg_sim} and \ref{fig:global_avg_interf}), task similarities appear deceptively high ($>0.91$), and the pairwise interference appears mild ($4\%-11\%$). 

However, we find that evaluating these metrics globally hides severe local conflicts due to a distinct, non-uniform layer-wise distribution. When we rank the layers and isolate the top-6 worst-affected layers (Figures \ref{fig:top6_avg_sim} and \ref{fig:top6_avg_interf}), a starkly different geometric landscape emerges. In these critical foundational layers, the similarity ($Sim$) drops significantly to the $0.75-0.87$ range, and the Hidden Interference ($Inf$) skyrockets, indicating that explicit updates from one task can overwrite up to $25\%$ of the implicit pre-trained dependencies of another (e.g., SUN397 and Cars invading SVHN's space). 

\begin{table}[htbp]
\centering
\begin{tabular}{c|ccc>{\columncolor{pink!30}}c}
\toprule
Task t& baseline $\theta_t$ & $\theta_{abla}\ (\Delta_{acc}^{LBW})$ & $\theta_{rand}\ (\Delta_{acc}^{rand})$ & $\Delta_{acc}^{LBW}/\Delta_{acc}^{rand}$ \\
\midrule
CIFAR100 & 91.08 & 1.88 (89.20) & 88.15 (2.93) & 30.48$\times$\\

CIFAR10 & 98.42 & 10.82 (87.60) & 97.89 (0.53) & 166.33$\times$\\

Cars & 86.49 & 0.37 (86.12) & 80.22 (6.27) & 13.75$\times$\\

DTD & 98.40 & 3.19 (95.21) & 98.05 (0.35) & 268.50$\times$\\

EMNIST & 99.78 & 17.00 (82.78) & 99.75 (0.03) & 3104.25$\times$\\

FER2013 & 44.43 & 16.34 (28.08) & 43.47 (0.95) & 29.49$\times$\\

FashionMNIST & 95.28 & 9.78 (85.50) & 95.14 (0.14) & 610.71$\times$\\

Flowers102 & 97.06 & 0.98 (96.08) & 91.18 (5.88) & 16.33$\times$\\

Food101 & 89.56 & 0.96 (88.60) & 85.99 (3.57) & 24.84$\times$\\

GTSRB & 99.92 & 1.91 (98.01) & 99.86 (0.06) & 1566.60$\times$\\

KMNIST & 99.94 & 11.92 (88.02) & 99.88 (0.06) & 1467.00$\times$\\

OxfordIIITPet & 94.84 & 2.17 (92.66) & 93.75 (1.09) & 85.25$\times$\\

PCAM & 97.86 & 49.10 (48.76) & 96.76 (1.10) & 44.33$\times$\\

RESISC45 & 97.04 & 4.18 (92.86) & 95.98 (1.06) & 87.75$\times$\\

RenderedSST2 & 77.31 & 54.91 (22.40) & 75.63 (1.69) & 13.29$\times$\\

STL10 & 99.60 & 11.00 (88.60) & 99.33 (0.27) & 332.25$\times$\\

SUN397 & 81.32 & 0.56 (80.77) & 76.19 (5.13) & 15.74$\times$\\

SVHN & 96.76 & 7.40 (89.36) & 96.51 (0.25) & 352.74$\times$\\
\bottomrule
\end{tabular}
\caption{$\Delta_{acc}^{*}$ represents the accuracy (\%) change due to the subspace filtering. The random subspace filtering is averaged across three random seeds. The numbers in parentheses indicate the decrease in accuracy compared to the $\theta_t$.}
\label{tab:attk3_detail}
\end{table}

This sharp contrast reveals that the degree of pairwise task conflict is highly dependent on layer depth, with early foundational layers absorbing severe, high-density interference. This pronounced layer-wise variation suggests that a uniform, global merging approach is insufficient, motivating a dedicated investigation into how these conflicts are distributed layer-by-layer across the pre-trained architecture.
\begin{figure}[htbp]
   \centering
   \begin{minipage}{0.48\linewidth}
       \centering
       \includegraphics[width=0.98\linewidth]{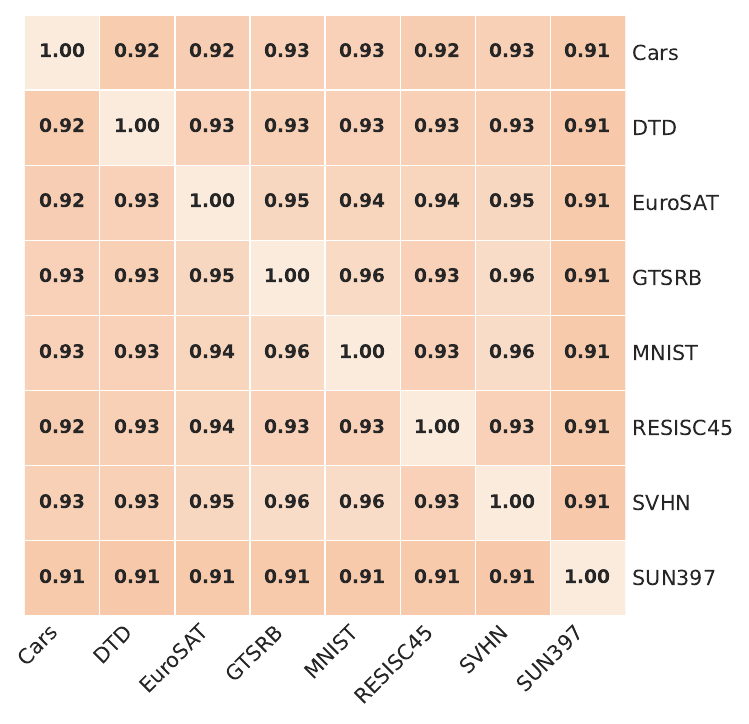}
       \caption{Sim (Global Avg)}
       \label{fig:global_avg_sim}
   \end{minipage}
   \hfill
   \begin{minipage}{0.48\linewidth}
       \centering
       \includegraphics[width=0.98\linewidth]{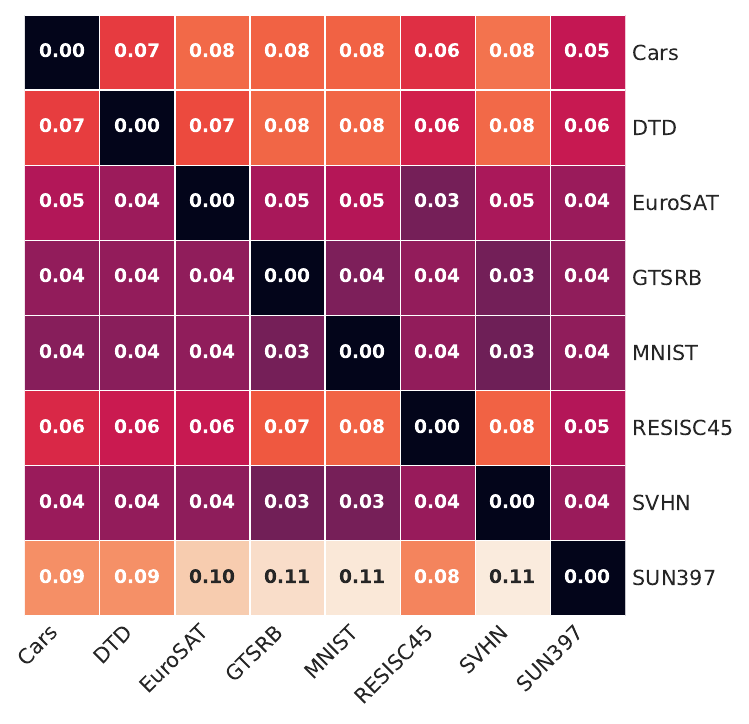}
       \caption{Interf (Global Avg)}
       \label{fig:global_avg_interf}
   \end{minipage}
   \\ \vspace{1em}
   \begin{minipage}{0.48\linewidth}
       \centering
       \includegraphics[width=0.98\linewidth]{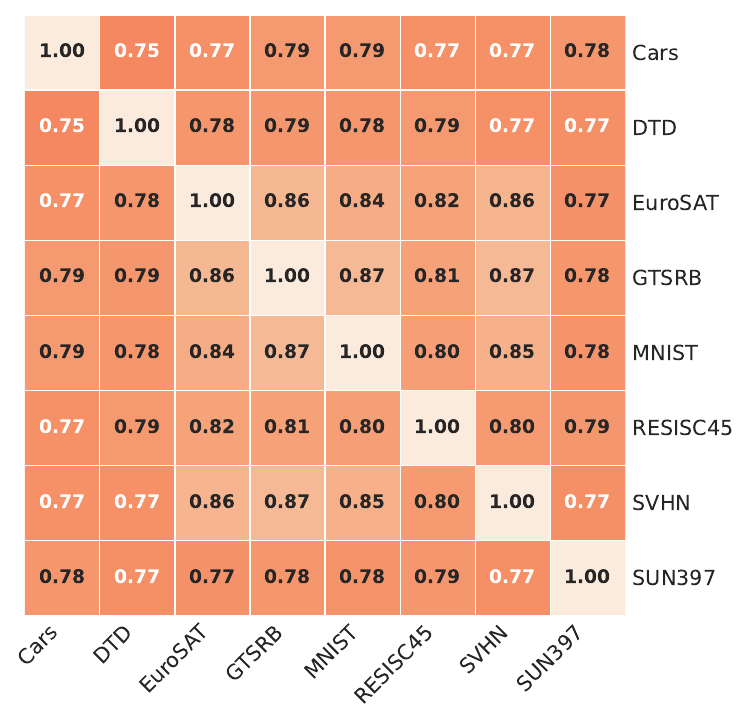}
       \caption{Sim (Top-6 Layers)}
       \label{fig:top6_avg_sim}
   \end{minipage}
   \hfill
   \begin{minipage}{0.48\linewidth}
       \centering
       \includegraphics[width=0.98\linewidth]{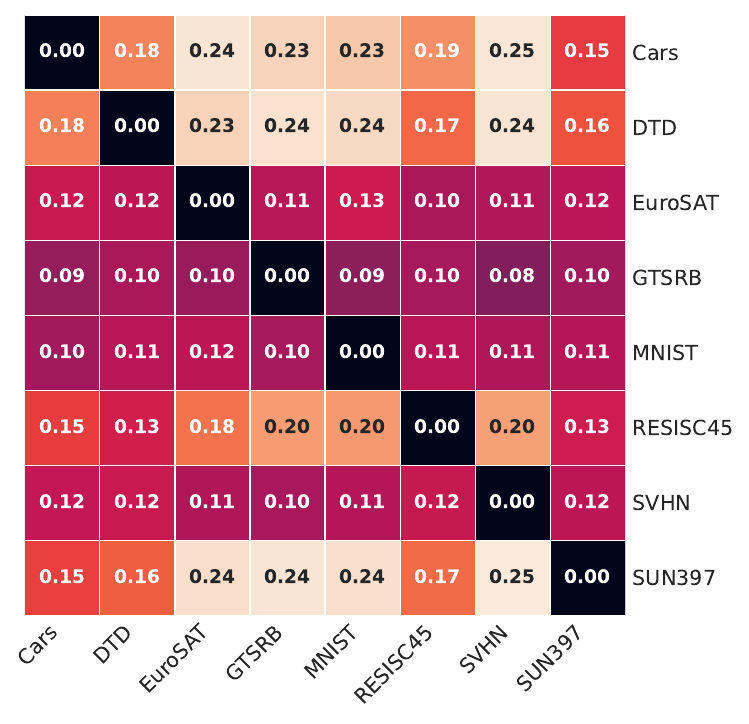}
       \caption{Interf (Top-6 Layers)}
       \label{fig:top6_avg_interf}
   \end{minipage}
   \caption{Comparison of global averages versus the top-6 worst-affected layers for Sacred Space Similarity (Sim) and Hidden Interference Ratio (Inf).}
   \label{fig:averaging_illusion_quad}
\end{figure}

\subsection{Layer-wise and Component-wise Analysis of Pre-trained Subspace Intrusion}
\label{app:layer_wise_deep_dive}

To systematically investigate how task updates interact with the pre-trained architecture under a subspace geometric view, we conduct a layer-wise analysis. This investigation requires a metric that directly quantifies how task-specific updates modify the shared, pre-trained representation space, rather than measuring pairwise task-to-task relationships. To this end, motivated by the projection strength metrics~\citep{saha2021gradient} and subspace alignment formulations~\citep{fernando2013unsupervised}, we define the \textbf{Intrusion Energy ($E_{in}$)} for a specific task $t$ at layer $\ell$ as:
\begin{align}\label{energy}
    E_{in}(t,\ell)=\frac{\|\Delta_tV_{0,K}\|_F^2}{\|\Delta_t\|_F^2}
\end{align}
where $V_{0,K} \in \mathbb{R}^{n \times K}$ represents the top-$K$ right singular vectors of the pre-trained weights $\theta_0$, capturing the core coordinates of general pre-trained knowledge. The objective of $E_{in}$ is to measure the exact proportion of a task's explicit update energy ($\Delta_t$) that projects directly onto this shared pre-trained core. A higher $E_{in}$ indicates that fine-tuning has aggressively reshaped the foundational pre-trained representation, creating a high-energy "intrusion" that is vulnerable to overwrite when merging.

\begin{figure}[htbp]
   \centering
   \includegraphics[width=0.75\linewidth]{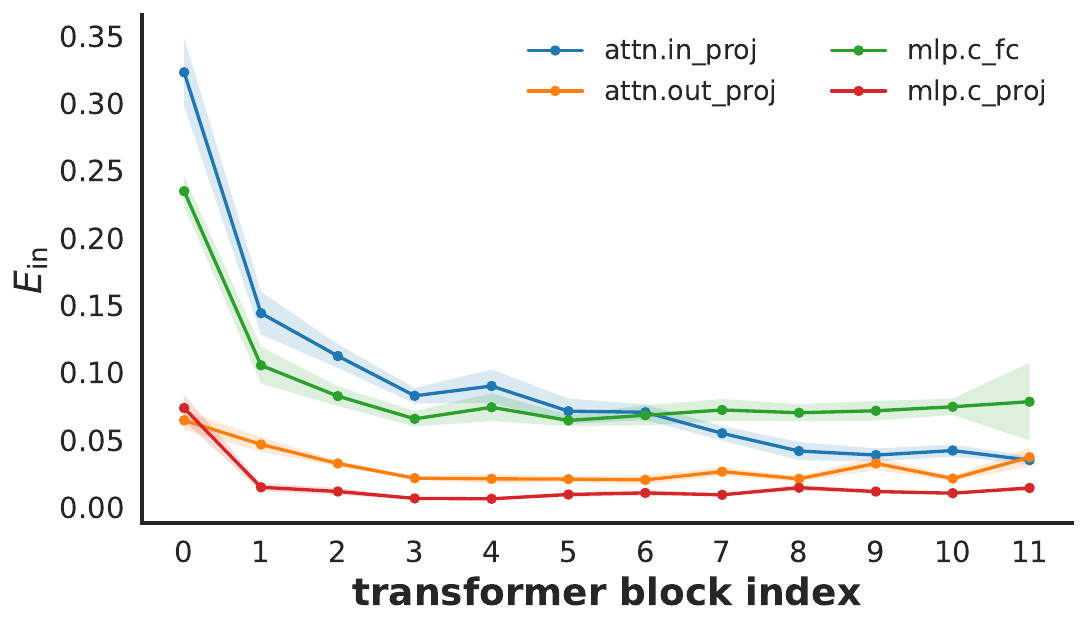} 
   \caption{Layer-wise Intrusion Energy $E_{in}$ across the transformer blocks of ViT-B-16.} 
   \label{fig:layer-wise energy} 
\end{figure}

\subsubsection{Global Layer-wise Intrusion Dynamics}
Figure \ref{fig:layer-wise energy} plots the average $E_{in}$ across the transformer blocks of ViT-B-16, revealing a clear layer-wise heterogeneity:
\paragraph{Foundational Bottlenecks in Early Blocks.} In early blocks (blocks 0 to 2), which extract fundamental visual representations (e.g., textures, edges), $E_{in}$ is heavily elevated. For instance, the intrusion energy in the input projection layers exceeds $30\%$. This indicates that downstream specialization forces task updates to aggressively modify these early representation bottlenecks, rendering them highly sensitive to cross-task interference.
\paragraph{Orthogonal Routing in Deeper Blocks:} Conversely, in deeper layers (blocks 5 to 11), $E_{in}$ drops significantly and plateaus below $5\%$. This confirms that deeper layers safely route high-level semantic updates into task-specific orthogonal subspaces with minimal perturbation of the core general-purpose representations.

This layer-wise distribution highlights the risk of ``the averaging illusion'': a $48\%$ structural corruption at foundational block 0 cannot be resolved by deeper blocks. Treating all layers uniformly allows high-energy foundational collisions to proceed unchecked, leading to downstream performance degradation.

\begin{figure}[htbp]
   \centering
   \includegraphics[width=0.9\linewidth]{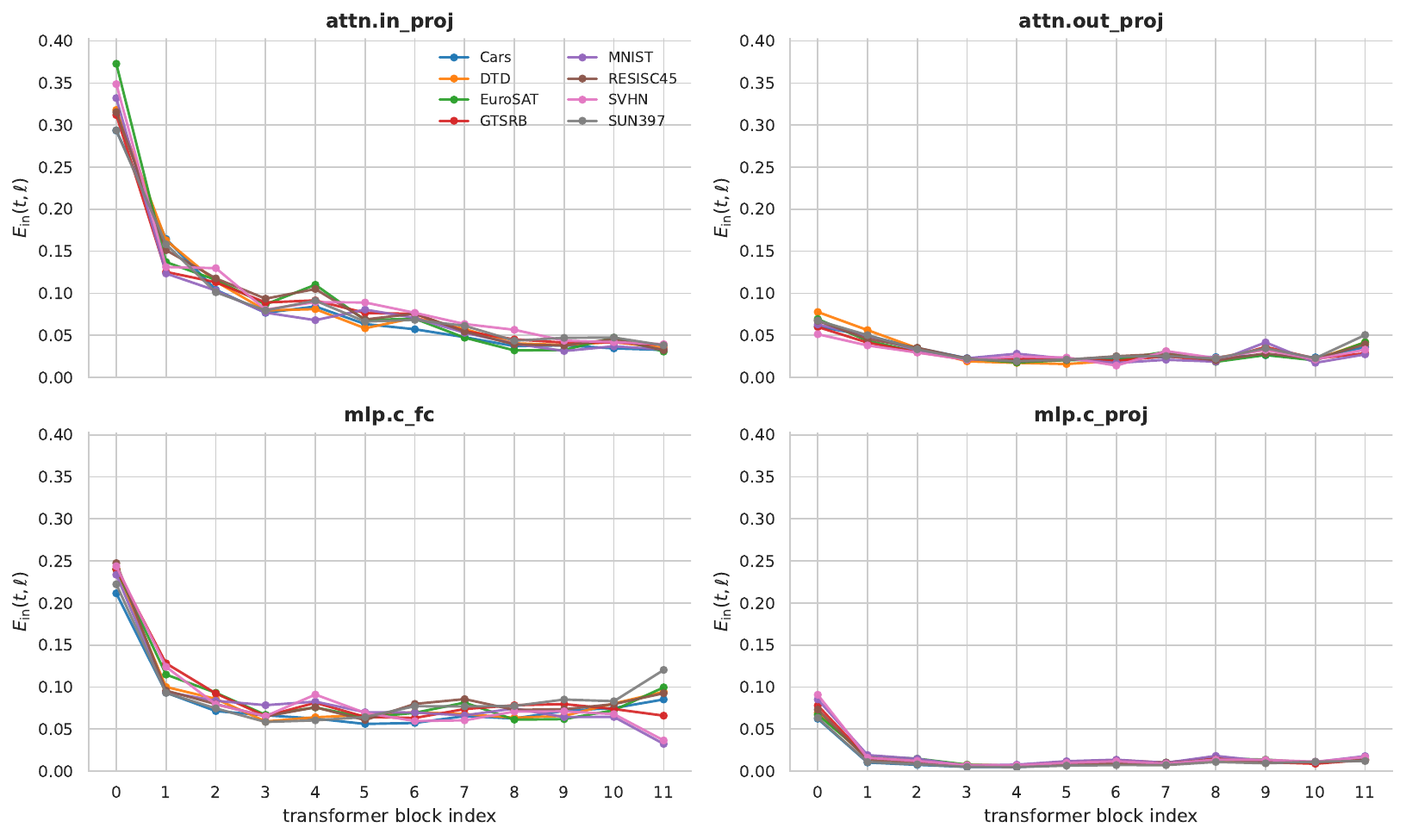} 
   \caption{Task-specific and Component-wise Intrusion Energy $E_{in}$ across different transformer components.} 
   \label{fig:ein_task_component_line} 
\end{figure}

\subsubsection{Component-wise and Task-wise Intrusion Dynamics}
To evaluate these dynamics at a more granular scale, Figure \ref{fig:ein_task_component_line} decomposes $E_{in}$ across individual tasks and specific transformer components (\texttt{attn.in\_proj}, \texttt{attn.out\_proj}, \texttt{mlp.c\_fc}, and \texttt{mlp.c\_proj}). 

We observe a strong component-wise heterogeneity. The input-facing projection layers (\texttt{attn.in\_proj} and \texttt{mlp.c\_fc}) absorb nearly $40\%$ of the update energy at block 0, acting as the primary bottleneck of task intrusion. In contrast, the output projection layers (\texttt{attn.out\_proj} and \texttt{mlp.c\_proj}) remain flat and low (mostly below $5\%$), confirming they act primarily as feature routers that do not perturb the shared core. Importantly, the tightly clustered curves across 8 diverse downstream datasets demonstrate that this early-layer collision is a universal geometric characteristic of ViT fine-tuning rather than an artifact of a specific task.

\begin{figure}[htbp]
   \centering
   \includegraphics[width=0.85\linewidth]{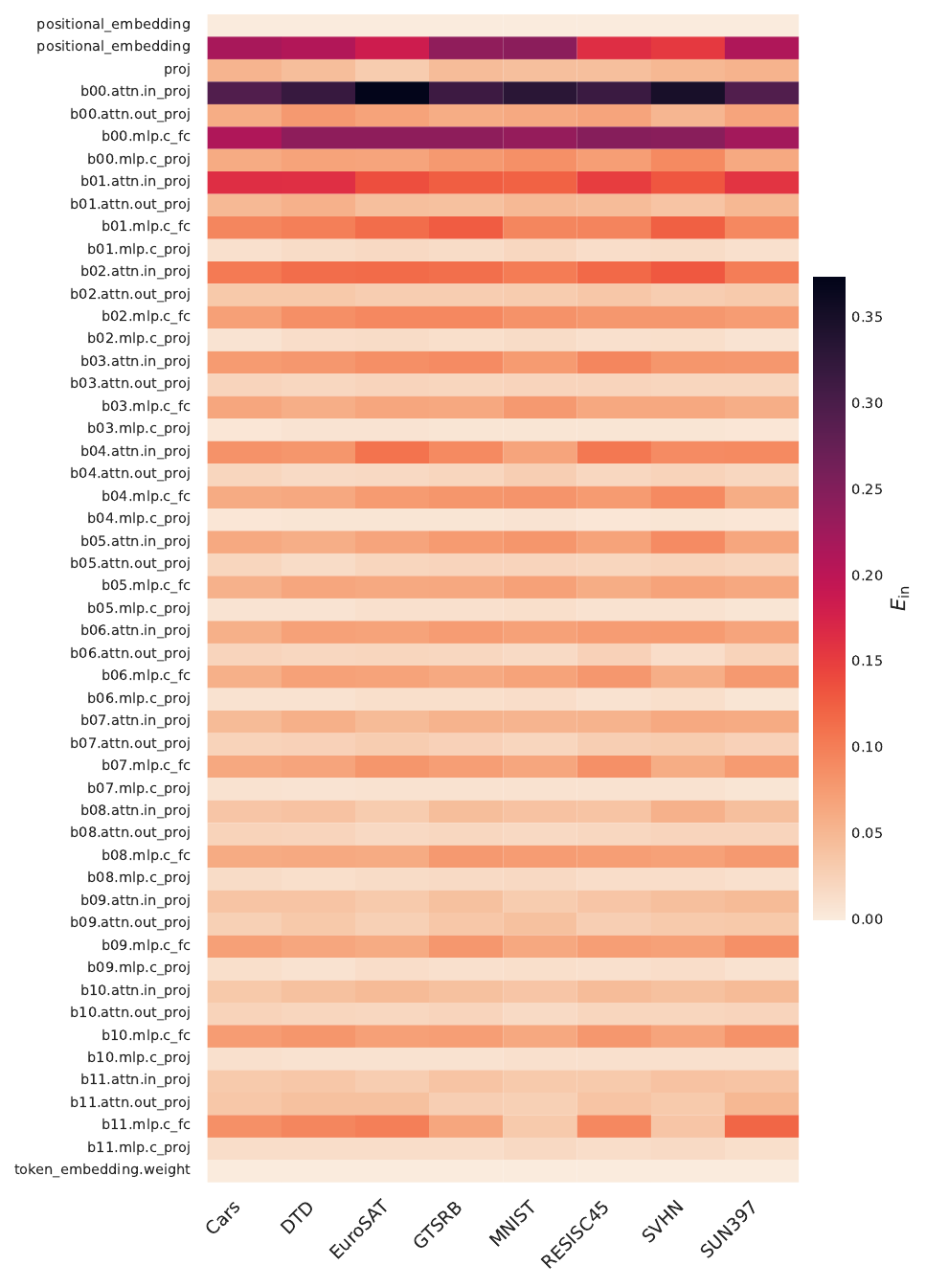} 
   \caption{Comprehensive heatmap of Intrusion Energy ($E_{in}$) across all layers and tasks.} 
   \label{fig:ein_full_heatmap} 
\end{figure}

This universality is further illustrated in the exhaustive heatmap in Figure \ref{fig:ein_full_heatmap}. Across all evaluated tasks, the earliest blocks (e.g., \texttt{b00.attn.in\_proj}, \texttt{b00.mlp.c\_fc}) manifest as prominent dark bands of intense, high-energy intrusion into the pre-trained space. Beyond block 4, the heatmap transitions consistently to light hues, confirming that the pre-trained core remains undisturbed in deeper blocks.

\begin{figure}[htbp]
   \centering
   \includegraphics[width=1.0\linewidth]{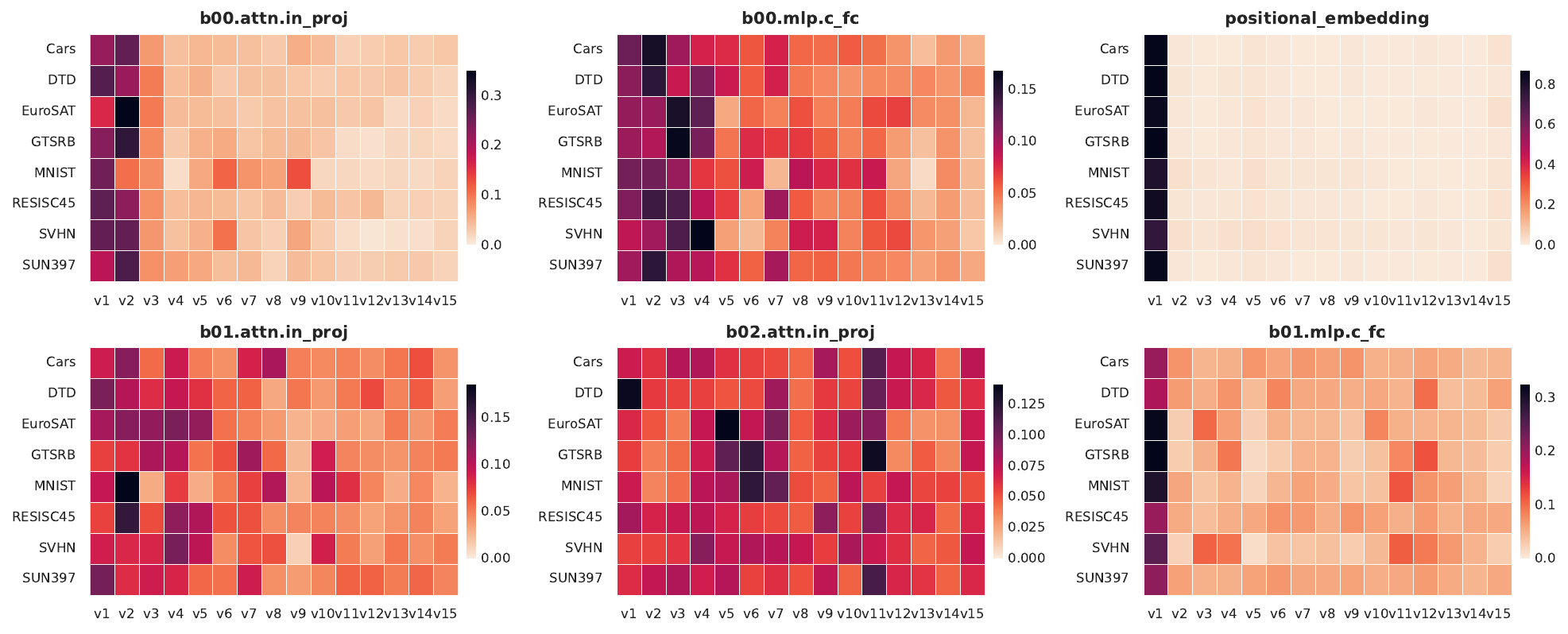} 
   \caption{Fractional Intrusion Energy distributed across individual right singular vectors ($v_1$ to $v_{15}$) for select top-affected layers.} 
   \label{fig:singular_vector_dist} 
\end{figure}

\subsubsection{Microscopic Intrusion Distribution Along Principal Axes}
To understand the geometric mechanics of this intrusion at the finest scale, we analyze which specific singular vectors within the pre-trained core subspace ($K=15$) are affected. Figure \ref{fig:singular_vector_dist} visualizes the fractional intrusion energy distributed across the individual right singular vectors ($v_1$ to $v_{15}$) for the most heavily affected layers. 

The visualization reveals that task updates do not uniformly perturb the core space. Instead, they target specific principal axes in a highly task-dependent manner. For example, in the foundational layer \texttt{b00.mlp.c\_fc}, SVHN heavily modifies $v_4$, whereas EuroSAT targets $v_3$. Additionally, in the \texttt{positional\_embedding} layer, the intrusion across all tasks is almost exclusively concentrated on the single most dominant vector ($v_1$), indicating a universal shift in positional priors during downstream specialization. 

This vector-wise heterogeneity highlights the necessity of \texttt{PACT}'s design. Because different tasks rely on and explicitly attack different combinations of principal axes, simply discarding a fixed, task-agnostic set of singular vectors is insufficient. The adaptive, QR-decomposition-based global shield implemented in \texttt{PACT} is essential to dynamically map and protect these highly heterogeneous, task-specific attack vectors.
\section{Detailed Ablation Studies}
\label{app:ablation_studies}

To systematically evaluate the contribution of each design component in \texttt{PACT}, we conduct extensive ablation studies. This section is organized into three parts: (1) an algorithmic component ablation comparing \texttt{PACT} against a static core-filtering baseline; (2) a layer-wise depth ablation exploring the performance of localized filtering; and (3) a sensitivity analysis of the hyper-parameters.

\subsection{Algorithmic Component Ablation: Dynamic Task Shields vs. Static \texttt{PACT} (S-\texttt{PACT})}
\label{app:algorithmic_ablation}

To verify the necessity of dynamically extracting the active task subspaces (the top-$k$ task-specific SVD) and constructing the mutual protection shields, we compare \texttt{PACT} against a simplified, static baseline which we term S-\texttt{PACT} (Static \texttt{PACT}). 

S-\texttt{PACT} bypasses the task-vector SVD and mutual projection phases entirely. Instead, it only performs a single SVD on the pre-trained weight matrix $\theta_0$ to extract the top-$K$ core directions $V_{0,K}$, and then directly projects each task matrix onto its orthogonal complement: 
\begin{align}
    \tilde{\Delta}_j = \Delta_j (I - V_{0,K} V_{0,K}^\top)
\end{align}
This baseline effectively filters task updates against the pre-trained foundation but ignores the distinct, task-specific subspaces of other expert models. S-\texttt{PACT} is computationally cheaper as it requires only one SVD per layer and no subsequent QR decompositions.

\begin{table}[htbp]
\centering
\scriptsize   
\setlength{\tabcolsep}{4pt}

\begin{tabular}{lccccccccc}
\toprule
\multirow{2}{*}{Method} & \multicolumn{3}{c}{ViT-B/32} & \multicolumn{3}{c}{ViT-B/16} & \multicolumn{3}{c}{ViT-L/14} \\
\cmidrule(lr){2-4} \cmidrule(lr){5-7} \cmidrule(lr){8-10}
 & 8 tasks & 14 tasks & 20 tasks & 8 tasks & 14 tasks & 20 tasks & 8 tasks & 14 tasks & 20 tasks \\

\midrule
TA  & 70.8(76.5) & 65.3(72.1) & 60.5(66.8) & 75.4(79.6) & 70.5(75.9) & 65.8(70.8) & 84.9(88.7) & 79.4(84.0) & 74.0(78.1) \\
S-\texttt{PACT-TA}  & 70.8(76.3) & 65.8(72.7) & 61.7(68.3) & 76.2(80.5) & 71.2(76.7) & 66.4(71.6) & 85.0(88.7) & 80.3(85.1) & 74.7(78.9) \\
\texttt{PACT-TA}   & 75.2(81.0) & 66.4(73.3) & 59.9(66.4) & 85.5(90.3) & 76.8(82.7) & 68.7(74.0) & 92.8(96.9) & 87.0(92.0) & 83.0(87.7) \\

\midrule
Iso-C & 86.6(93.3) & 79.5(87.3) & 75.1(82.2) & 90.8(96.0) & 84.4(90.7) & 79.7(85.4) & 94.7(98.9) & 89.2(94.3) & 87.7(92.4) \\
S-\texttt{PACT-Iso-C} & 86.5(93.2) & 79.5(87.2) & 74.8(81.7) & 90.8(96.0) & 84.2(90.4) & 79.7(85.6) & 94.7(98.9) & 89.3(94.4) & 87.7(92.4) \\
\texttt{PACT-Iso-C}   & 87.3(94.1) & 81.3(89.0) & 78.7(86.0) & 92.0(97.2) & 86.1(92.3) & 83.4(89.3) & 95.1(98.3) & 90.5(95.5) & 90.3(95.0) \\

\bottomrule
\end{tabular}
\caption{Algorithmic ablation study comparing \texttt{PACT} against the core-only S-\texttt{PACT} baseline across different model backbones. We present average absolute accuracy and average normalized accuracy (in bracket) in $\%$.}
\label{tab:s-pact_ablation}
\end{table}

Table \ref{tab:s-pact_ablation} summarizes the merging performance of S-\texttt{PACT} against TA, Iso-C, and full \texttt{PACT}. 
The empirical results reveal that merely performing static core-filtering (S-\texttt{PACT}) yields limited performance improvements. For instance, on ViT-B/16 (14 tasks), S-PACT-TA only marginally improves over TA ($71.2\%$ vs. $70.5\%$), while S-\texttt{PACT-Iso-C} slightly degrades the original Iso-C baseline ($84.2\%$ vs. $84.4\%$). 

In contrast, \texttt{PACT}-TA and \texttt{PACT-Iso-C} achieve substantial performance improvements, reaching $76.8\%$ and $86.1\%$ respectively on the same setting. This difference indicates that simply protecting the static pre-trained core is insufficient. To prevent task updates from blindly overwriting each other's distinct downstream manifolds, it is critical to dynamically extract the active task dimensions and construct mutually orthogonal protection shields.

\subsection{Layer-wise Depth Ablation: Localized Filtering for Efficiency}
\label{app:layer_ablation}

In Appendix~\ref{app:layer_wise_deep_dive}, we demonstrated that the energy intrusion effect is highly localized within the early foundational layers of the model, whereas deeper layers route semantics in mutually orthogonal directions. Based on this observation, we perform an ablation on the depth of layers filtered by \texttt{PACT} to identify potential efficiency-performance trade-offs. We compare applying \texttt{PACT} globally to $100\%$ of the layers versus applying it to various percentages of early foundational layers.

\begin{table}[htbp]
\centering
\scriptsize   
\setlength{\tabcolsep}{4pt}
\begin{tabular}{lccccccccc}
\toprule
\multirow{2}{*}{Method} & \multicolumn{3}{c}{ViT-B/32} & \multicolumn{3}{c}{ViT-B/16} & \multicolumn{3}{c}{ViT-L/14} \\
\cmidrule(lr){2-4} \cmidrule(lr){5-7} \cmidrule(lr){8-10}
 & 8 tasks & 14 tasks & 20 tasks & 8 tasks & 14 tasks & 20 tasks & 8 tasks & 14 tasks & 20 tasks \\

\midrule
TA  & 70.8(76.5) & 65.3(72.1) & 60.5(66.8) & 75.4(79.6) & 70.5(75.9) & 65.8(70.8) & 84.9(88.7) & 79.4(84.0) & 74.0(78.1) \\
25$\%$ PACT-TA  & 72.8(78.6) & \textbf{67.3(74.3)} & \textbf{62.3(68.9)} & 77.9(82.3) & 72.5(78.0) & 66.9(72.1) & 86.5(90.2) & 81.7(86.5) & 76.1(80.3) \\
100$\%$ PACT-TA   & \textbf{75.2(81.0)} & 66.4(73.3) & 59.9(66.4) & \textbf{85.5(90.3)} & \textbf{76.8(82.7)} & \textbf{68.7(74.0)} & \textbf{92.8(96.9)} & \textbf{87.0(92.0)} & \textbf{83.0(87.7)} \\

\midrule
TSV-M         & 85.9(92.3) & 80.1(87.9) & 77.1(84.3) & 89.0(93.9) & 84.6(91.0) & 80.6(86.5) & 93.0(97.0) & 89.2(94.4) & 87.7(92.5) \\
25$\%$ PACT-TSVM   & 86.5(93.1) & 79.8(87.2) & 77.2(84.2) & 89.5(94.4) & 83.7(89.8) & 80.7(86.3) & 93.6(97.6) & 88.0(92.9) & 87.6(92.2) \\
100$\%$ PACT-TSVM   & \textbf{88.1(94.8)} & \textbf{82.0(89.6)} & \textbf{80.3(87.6)} & \textbf{91.2(96.3)} & \textbf{85.4(91.5)} & \textbf{83.8(89.6)} & \textbf{94.4(98.5)} & \textbf{89.5(94.4)} & \textbf{89.7(94.3)} \\

\midrule
Iso-C & 86.6(93.3) & 79.5(87.3) & 75.1(82.2) & 90.8(96.0) & 84.4(90.7) & 79.7(85.4) & 94.7(98.9) & 89.2(94.3) & 87.7(92.4) \\
25$\%$ PACT-IsoC & 86.9(93.6) & 80.0(87.7) & 75.6(82.7) & 90.9(96.1) & 84.7(90.9) & 80.0(85.8) & 94.8(99.0) & 89.4(94.5) & 88.0(92.7) \\
50$\%$ PACT-IsoC & \textbf{87.8(94.7)} & 81.8(89.6) & 78.8(85.9) & 91.6(96.9) & 85.8(92.1) & 82.2(88.0) & 95.0(99.2) & 90.0(95.1) & 89.2(93.9) \\
75$\%$ PACT-IsoC & 87.7(94.4) & \textbf{82.2(90.1)} & \textbf{79.1(86.2)} & \textbf{92.1(97.4)} & \textbf{86.5(92.9)} & 83.2(89.1) & \textbf{95.2(99.3)} & \textbf{90.5(95.6)} & 90.0(94.8) \\
100$\%$ PACT-IsoC   & 87.3(94.1) & 81.3(89.0) & 78.7(86.0) & 92.0(97.2) & 86.1(92.3) & \textbf{83.4(89.3)} & 95.1(98.3) & 90.5(95.5) & \textbf{90.3(95.0)} \\

\bottomrule
\end{tabular}
\caption{Ablation study on filtering depth comparing localized early-layer \texttt{PACT} filtering under various percentages versus $100\%$ global \texttt{PACT} filtering. We present average absolute accuracy and average normalized accuracy (in bracket) in $\%$. Best performances are shown in \textbf{bold}.}
\label{tab:layer_ablation}
\end{table}

Table \ref{tab:layer_ablation} presents the comparison across three base merging algorithms (TA, Iso-C, TSV-M). Overall, integrating \texttt{PACT} systematically improves the multi-task accuracy of all three baseline methods across all evaluated backbones and task scales. Specifically, while baseline methods suffer from performance degradation as the task load increases, the \texttt{PACT}-enhanced variants successfully mitigate this drop, with the best-performing configurations generally concentrated at $75\%$ or $100\%$ filtering depths. To rigorously trace the depth-wise trajectory of this base-protection effect, we focus our analysis primarily on \texttt{PACT-IsoC}, which provides a complete gradient of filtering depths ($25\%$, $50\%$, $75\%$, and $100\%$).

For \texttt{PACT-IsoC}, we observe a non-monotonic performance curve as the filtering depth increases. Moving from the baseline Iso-C to $25\%$ and $50\%$ \texttt{PACT-IsoC}, the accuracy steadily climbs across all configurations, confirming that shielding the load-bearing walls in the early-to-mid layers is highly effective. Interestingly, the performance peaks at different depths depending on model capacity and task scale. For the majority of 8-task and 14-task scenarios on ViT-B/16 and ViT-L/14, the accuracy peaks at $75\%$ depth (e.g., reaching $92.1\%$ and $95.2\%$ on 8 tasks) and slightly plateaus or declines at $100\%$. This indicates that while protecting the first $75\%$ of the layers is essential due to high-energy visual bottlenecks, the deepest $25\%$ of layers route high-level semantic representations in naturally orthogonal directions, making strict orthogonal constraints in these final blocks redundant or over-constraining.

This peak behavior dynamically shifts when scaling the backbone capacity or task load. For the highly parameter-constrained ViT-B/32 backbone, the optimal performance peaks earlier---at $50\%$ depth for 8 tasks ($87.8\%$) and at $75\%$ depth for more tasks. Conversely, under the heaviest task load of 20 tasks on larger backbones (ViT-B/16 and ViT-L/14), the peak shifts to $100\%$ global filtering (achieving $83.4\%$ and $90.3\%$). This shift reveals that under extreme multi-task interference, the necessity of safeguarding the pre-trained core extends globally across all layers, and larger backbones possess sufficient dimensional capacity to accommodate the global orthogonal constraints without suffering from dimensionality lock.

\subsection{Hyperparameter Sensitivity Analysis}
\label{app:hyperparameter_ablation}

\begin{table}[htbp]
\centering
\scriptsize   
\setlength{\tabcolsep}{4pt}
\begin{tabular}{cccccccccc}
\toprule
\multirow{2}{*}{($K$, $k$)} & \multicolumn{3}{c}{ViT-B/32} & \multicolumn{3}{c}{ViT-B/16} & \multicolumn{3}{c}{ViT-L/14} \\
\cmidrule(lr){2-4} \cmidrule(lr){5-7} \cmidrule(lr){8-10}
 & 8 tasks & 14 tasks & 20 tasks & 8 tasks & 14 tasks & 20 tasks & 8 tasks & 14 tasks & 20 tasks \\

\midrule
(0,0) & 86.6(93.3) & 79.5(87.3) & 75.1(82.2) & 90.8(96.0) & 84.4(90.7) & 79.7(85.4) & 94.7(98.9) & 89.2(94.3) & 87.7(92.4) \\
$\alpha$ &1.40 &0.90 &0.90 &1.40 &1.00 &0.80 &1.60 &1.10 &1.00 \\

\midrule
(38,32) & 86.2(92.9) & 79.5(88.1) & 70.2(77.1) & 91.8(97.0) & 85.5(91.9) & 80.0(85.8) & 95.2(99.4) & 90.4(95.5) & 90.2(94.9) \\
$\alpha$ &2.30 &2.80 &4.30 &2.50 &2.70 &5.20 &2.40 &2.00 &3.20 \\

\midrule
(30,24) & 86.6(93.4) & 80.4(88.1) & 76.6(83.7) & 91.9(97.1) & 86.1(92.4) & 82.9(88.8) & 95.2(99.4) & 90.4(95.5) & 90.2(94.9) \\
$\alpha$ & 2.00&2.30 & 3.00&2.20 &2.30 &3.10 &2.40 &1.80 &2.30 \\

\midrule
(23,16) & 87.1(93.9) & 81.2(88.8) & 78.2(85.4) & 92.0(97.2) & \textbf{86.4(92.7)} & 83.1(89.0) & \textbf{95.2(99.4)} & 90.5(95.6) & 90.2(95.0) \\
$\alpha$ &1.90 &1.90 &2.20 &2.00 &1.80 &2.30 &2.20 &1.60 &1.90 \\

\midrule
(15,8)   & \textbf{87.3(94.1)} & \textbf{81.3(89.0)} & \textbf{78.7(86.0)} & \textbf{92.0(97.2)} & 86.1(92.3) & \textbf{83.4(89.3)} & 95.1(98.3) & \textbf{90.5(95.5)} & \textbf{90.3(95.0)} \\
$\alpha$ &1.70 &1.70 &1.60 &1.80 &1.60 &1.60 &2.00 &1.50 &1.50 \\

\bottomrule
\end{tabular}
\caption{Performance of \texttt{PACT-Iso-C} under different hyperparameters ($K$, $k$) and the corresponding scaling coefficient, where (0,0) represents the results of our reproduced Iso-C. We present average absolute accuracy and average normalized accuracy (in brackets) in $\%$. The best performance is in \textbf{bold}.}
\label{tab:abla}
\end{table}

To study the sensitivity of \texttt{PACT} to its primary hyperparameters, we refer to the detailed evaluation presented in Table~\ref{tab:abla}. Table~\ref{tab:abla} examines the effects of varying the pre-trained core dimension $K$ and the active task vector dimension $k$, alongside the corresponding base merging scaling coefficients $\alpha$.

The sensitivity analysis indicates that \texttt{PACT} is relatively robust to different hyperparameter configurations. Specifically, the parameter pair $(K, k) = (15, 8)$ consistently yields favorable absolute and normalized accuracy across diverse dataset scales (8, 14, and 20 tasks) and various model backbones (ViT-B/32, ViT-B/16, and ViT-L/14). This stable behavior suggests that a moderate, low-rank projection configuration is sufficient to successfully segregate the shared general-purpose representation from the task-specific update manifolds, simplifying practical deployments.

\section{Experimental Details} \label{sec:exp_detail}
In this Appendix, we provide the dataset and implementation details used to carry out the experiments presented in the paper.

\subsection{Datasets}

The 8-dataset benchmark consists of: Cars~\citep{krause20133d}, DTD~\citep{cimpoi2014describing}, EuroSAT~\citep{helber2019eurosat}, GTSRB~\citep{stallkamp2011german}, MNIST~\citep{lecun1998gradient}, RESISC45~\citep{cheng2017remote}, SUN397~\citep{xiao2016sun}, and SVHN~\citep{netzer2011reading}.

The 14-dataset benchmark builds on the preceding one, incorporating six additional datasets: CIFAR100~\citep{krizhevsky2009learning}, STL10~\citep{coates2011analysis}, Flowers102~\citep{nilsback2008automated}, OxfordIIITPet~\citep{parkhi2012vedaldi}, PCAM~\citep{veeling2018rotation}, and FER2013~\citep{goodfellow2013challenges}.

Finally, the 20-dataset benchmark includes the preceding 14 plus the following six: EMNIST~\citep{cohen2017emnist}, CIFAR10~\citep{krizhevsky2009learning}, Food101~\citep{bossard2014food}, FashionMNIST~\citep{xiao2017fashion}, RenderedSST2~\citep{socher2013recursive}, and KMNIST~\citep{ba2016layer}.

\begin{itemize}
    \item \textbf{SUN397} ~\citep{xiao2016sun} contains more than $100,000$ images of 397 categories for benchmarking scene understanding. The number of images varies across categories, but there are at least 100 images each.
    
    \item \textbf{Stanford Cars (Cars)} ~\citep{krause20133d} has $16,185$ images in total of 196 types of cars and is evenly split for training and testing sets.
    
    \item \textbf{RESISC45} ~\citep{cheng2017remote} is developed for remote sensing image scene classification. This dataset covers 45 scene classes with 700 images of size $256 \times 256$ for each.
    
    \item \textbf{EuroSAT}~\citep{helber2019eurosat} is used for land use and land cover classification using Sentinel-2 satellite images of size $64 \times 64$, consisting of $27,000$ images covering 10 classes.
    
    \item \textbf{SVHN}~\citep{netzer2011reading} is a street view house number classification benchmark, containing more than $600,000$ RGB images of 10 printed digits in size $32 \times 32$ cropped from house number plates.
    
    \item \textbf{GTSRB} ~\citep{stallkamp2011german} is a German traffic sign recognition benchmark consisting of over $50,000$ images of 43 classes of traffic signs in varying light and background conditions.
    
    \item \textbf{MNIST}~\citep{lecun1998gradient} is a well-known classical dataset for hand-written digit classification with $60,000$ training and $10,000$ testing images of size $28 \times 28$ in 10 classes of numbers.
    
    \item \textbf{DTD}~\citep{cimpoi2014describing} is a collection of $5,640$ images across 47 categories of textures in the wild, annotated with human-centric attributes.

    \item \textbf{Flowers102} ~\citep{nilsback2008automated} contains 102 flower categories that are popular in the United Kingdom, with $1,020$ training and $6,149$ testing images. The images have varying poses and light conditions.
    
    \item \textbf{PCAM (PatchCamelyon)} ~\citep{veeling2018rotation} consists of more than 300M color images in size of $96 \times 96$ pixels extracted from histopathologic scans of lymph node sections. Each of them is annotated with a binary class indicating the presence of metastatic tissue.
    
    \item \textbf{FER2013}~\citep{goodfellow2013challenges} is developed for facial expression recognition. The images are grayscale and have a size of $48 \times 48$ pixels, describing seven different kinds of emotions. The training and testing split consists of $28,709$ and $7,178$ samples, respectively.
    
    \item \textbf{OxfordIIITPet} ~\citep{parkhi2012vedaldi} is a 37-category pet dataset with roughly 200 images for each category, and is equally divided for both training and testing splits. The images vary in scale, pose, and lighting conditions.
    
    \item \textbf{STL10}~\citep{coates2011analysis} is primarily built for unsupervised image recognition tasks covering 10 classes. Hence, the number of labeled images is quite small: 500 training and 800 testing images for each class. All of them are in $96 \times 96$ pixel resolution.
    
    \item \textbf{CIFAR100}~\citep{krizhevsky2009learning} consists of color images categorized in 100 general classes, each class contains 600 images, and each image is in size $32 \times 32$. There are $50,000$ training images and $10,000$ testing images.
    
    \item \textbf{CIFAR10} ~\citep{krizhevsky2009learning} is similar to CIFAR100, except it has 10 classes.
    
    \item \textbf{Food101}~\citep{bossard2014food} contains of 101 food categories, with $101,000$ images. For each class, 750 images are for training and 250 are for testing. Only the testing images are manually reviewed. The training images contain noise mostly from intense colors, and sometimes are mislabelled.
    
    \item \textbf{FashionMNIST} ~\citep{xiao2017fashion} is designed as a drop-in replacement benchmark for the original MNIST, thereby inheriting the same structure as MNIST.
    
    \item \textbf{EMNIST} ~\citep{cohen2017emnist} is an extended version of MNIST. EMNIST contains images of both characters and digits. We choose to use only the EMNIST Letters split, which contains around $145,000$ images evenly distributed in 26 classes of the alphabet letters.
    
    \item \textbf{KMNIST} ~\citep{ba2016layer}, yet another version of MNIST, represents 10 Japanese Hiragana characters.
    
    \item \textbf{RenderedSST2} ~\citep{socher2013recursive,radford2019language} is used for evaluating the models' capability on optical character recognition. The images are rendered from sentences in the Stanford Sentiment Treebank v2 ~\citep{socher2013recursive}, with black texts on a white background in $448 \times 448$ resolution. Each image is labeled as positive or negative based on the mood expressed in the text, and the number of images for both classes is nearly balanced. There are $6,920$ training and $1,821$ testing images.
\end{itemize}

\subsection{Additional Notes}
Our method relies on SVD and RSVD, which is defined for two-dimensional matrices $\Delta \in \mathbb{R}^{m \times n}$. However, some weights of the neural networks are represented by vectors $\delta \in \mathbb{R}^n$, e.g. bias vectors and parameters of layer normalization~\citep{ba2016layer}. Therefore, these parameters bypass the filtering process of the \texttt{PACT} algorithm and directly enter the subsequent merging process.

\subsection{Selection of the scaling factor $\alpha$}
\begin{figure}[htbp]
   \centering
   \includegraphics[width=0.85\linewidth]{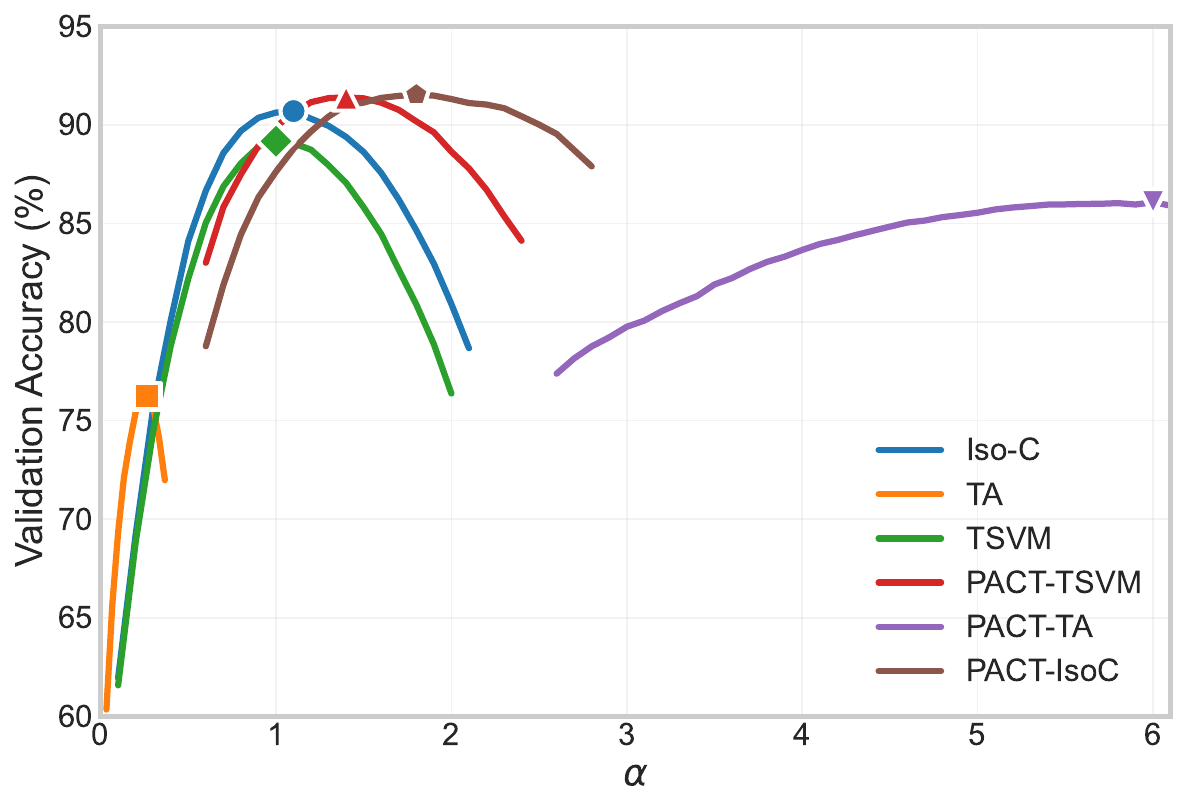}
   \caption{The $\alpha$ is chosen based on the best average performance on the validation set across tasks. Each point denotes the optimal $\alpha$ for each method. The model is ViT-B/16.} 
   \label{fig:scaling} 
\end{figure}
In Figure \ref{fig:scaling}, we present the relationship between validation accuracy and the scaling factor $\alpha$ on the 8-task ViT-B/16 model. TA is sensitive to the choice of $\alpha$, whereas Iso-C and TSV-M are relatively robust, which is consistent with previous findings. Due to the cleaning and filtering effect of \texttt{PACT}, the optimal $\alpha$ values increase for all three methods when combined with \texttt{PACT}, which also aligns with theoretical predictions. Notably, \texttt{PACT-TA} becomes robust to the choice of the optimal $\alpha$, though the value itself becomes substantially large. For reproducibility, in Table \ref{tab:alpha} we provide the optimal $\alpha$ value chosen on the held-out validation set for each model and number of tasks. It should be noted that, in the optimal $\alpha$ search procedures, the early stopping step is set to be 3.

\begin{table}[htbp]
\centering
\scriptsize   
\setlength{\tabcolsep}{4pt}

\begin{tabular}{cccccccccc}
\toprule
\multirow{2}{*}{method} & \multicolumn{3}{c}{ViT-B/32} & \multicolumn{3}{c}{ViT-B/16} & \multicolumn{3}{c}{ViT-L/14} \\
\cmidrule(lr){2-4} \cmidrule(lr){5-7} \cmidrule(lr){8-10}
 & 8 tasks & 14 tasks & 20 tasks & 8 tasks & 14 tasks & 20 tasks & 8 tasks & 14 tasks & 20 tasks \\

\midrule
TA &0.27 &0.10 &0.70 &0.27 &0.13 &0.10 &0.30 &0.17 &0.10 \\
\texttt{PACT-TA} &5.40 &5.30 &4.20 &6.00 &6.70 &8.10 &5.70 &7.10 &10.10 \\

\midrule
Iso-C &1.40 &0.90 &0.90 &1.40 &1.00 &0.80 &1.60 &1.10 &1.00 \\
\texttt{PACT-IsoC} &1.70 &1.70 &1.60 &1.80 &1.60 &1.60 &2.00 &1.50 &1.50 \\

\midrule
TSV-M &1.00 &1.00 &0.90 &1.00 &1.00 &0.90 &1.10 &1.00 &1.00 \\
\texttt{PACT-TSVM} &1.40 &1.50 &1.70 &1.40 &1.60 &1.80 &1.40 &1.70 &1.70 \\

\bottomrule
\end{tabular}
\caption{Optimal $\alpha$ value chosen on a held-out validation set for different model types and numbers of tasks.}
\label{tab:alpha}
\end{table}

\section{Limitation} \label{sec:limit}
Although \texttt{PACT} approaches model merging from a new perspective, namely the identification and protection of LBW dimensions, and achieves substantial improvements over existing methods, several limitations remain. First, despite the proposed efficient variant, \texttt{PACT} still introduces additional computational overhead. Second, there may exist more efficient and accurate strategies for identifying LBW dimensions. We leave these directions for future work.

\end{document}